\documentclass{article}

\usepackage[preprint]{neurips_2026}

\usepackage[utf8]{inputenc} 
\usepackage[T1]{fontenc}    
\usepackage{hyperref}       
\usepackage{url}            
\usepackage{booktabs}       
\usepackage{amsfonts}       
\usepackage{nicefrac}       
\usepackage{microtype}      
\usepackage{xcolor}         
\usepackage{amsmath}
\usepackage{amssymb}
\usepackage{amsthm}
\usepackage{graphicx}
\usepackage{enumitem}
\newtheorem{theorem}{Theorem}
\newtheorem{proposition}{Proposition}
\newtheorem{definition}{Definition}
\newtheorem{corollary}{Corollary}
\newtheorem{lemma}{Lemma}
\newtheorem{remark}{Remark}
\newtheorem{assumption}[theorem]{Assumption}
\newcommand{\Sd}{{\mathbb{S}^{d-1}}}
\newcommand{\VolS}{\text{Vol}\left(\Sd\right)}
\newcommand{\R}{\mathbb{R}}
\newcommand{\E}{\mathbb{E}}

\newcommand{\norm}[1]{\left\| #1 \right\|}

\newcommand{\search}{ p}
\newcommand{\growth}{ r}
\newcommand{\base}{\varphi}
\newcommand{\normC}{C_{\operatorname{norm}}}
\newcommand{\data}{\mathcal{F}}
\newcommand{\gdatamu}{g^\data_\mu}
\newcommand{\gdataS}{g^\data_\Sigma}
\newcommand{\gLmu}{g_\mu}
\newcommand{\gLS}{g_\Sigma}
\newcommand{\gLmuS}{g_{\mu,\Sigma}}
\newcommand{\compact}{\mathcal{K}}

\title{SGD for Variational Inference: Tackling Unbounded Variance via Preconditioning and Dynamic Batching}

\author{%
  Hippolyte Labarri\`ere \\
  MaLGa, DIBRIS\\
  Universit\`a degli Studi di Genova\\
  Genoa, Italy \\
  \texttt{hippolyte.labarriere@edu.unige.it} \\
  \And
  Cesare Molinari \\
  MaLGa, DIMA\\
  Universit\`a degli Studi di Genova\\
  Genoa, Italy \\
  \texttt{cesare.molinari@edu.unige.it} \\
  \And
  Silvia Villa \\
  MaLGa, DIMA\\
  Universit\`a degli Studi di Genova\\
  Genoa, Italy \\
  \texttt{silvia.villa@unige.it} \\
  \And
  Lorenzo Rosasco\\
  MaLGa, DIBRIS, IIT\\
  Universit\`a degli Studi di Genova\\
  Genoa, Italy \\
  \texttt{lorenzo.rosasco@unige.it}
}

\begin{document}

\maketitle

\begin{abstract}
Black-Box Variational Inference (BBVI) typically relies on Stochastic Gradient Descent (SGD) to optimize the Evidence Lower Bound (ELBO). However, the stochastic gradients in BBVI inherently exhibit unbounded variance, violating standard assumptions and instead satisfying the weaker Blum-Gladyshev (BG) condition, where variance grows quadratically with distance from the optimum. In this paper, we bridge the gap between stochastic optimization theory and the practical instances of BBVI. Focusing on the broad elliptic location-scale family of parameterized distributions, we offer two main  contributions. First, we prove the existence of an ELBO solution, a foundational property usually assumed a priori in the literature. Second, we establish comprehensive convergence guarantees spanning finite-time and asymptotic regimes for Minibatch Projected SGD (PSGD) equipped with dynamic batching and preconditioning under the BG condition. Our theoretical framework demonstrates that dynamic batching combined with preconditioning systematically enables rigorous guarantees even in complex settings. We illustrate our theoretical findings with numerical results, highlighting the efficacy of our approach for modern inference tasks.
\end{abstract}
\section{Introduction}
Variational Inference (VI) is a popular approach to approximate an intractable posterior distributions by casting Bayesian inference as an optimization problem \cite{Jordanetal1999, BleiKucukelbirMcAuliffe2017}. Given observed data $z \in \mathbb{R}^m$ and latent variables $x \in \mathbb{R}^d$, the posterior is given by the Bayes' rule $q(x|z) = q(x,z) / q(z)$. Since the marginal likelihood $q(z)$ is typically intractable, VI introduces a parameterized family of distributions $(p_\theta)_{\theta\in\Theta}$ and seeks to minimize the Kullback-Leibler (KL) divergence between the approximation and the true posterior. This is equivalent to minimizing the negative Evidence Lower Bound (ELBO):\begin{equation}\label{eq:pb_VI}\text{NegativeELBO}(\theta) = -\mathbb{E}_{x\sim p_\theta}\left[\log q(x,z) + \log p_\theta(x)\right].\end{equation}In this paper, we focus on Black-Box VI (BBVI) \cite{kingma2014efficient, ranganath2014black}, where the quantity $\mathbb{E}_{x\sim p_\theta}\left[\log q(x,z)\right]$ is possibly too complex to be differentiated and/or not fully known, and only stochastic gradients are available. While BBVI provides a scalable framework relying on stochastic gradients of the ELBO, developing  theoretical guarantees remains an ongoing challenge due to unbounded variance. Although alternative optimization methods like Natural Gradient Descent (NGD) \cite{amari1998natural} exist, their theoretical guarantees are largely restricted to exponential families or conjugate models \cite{wu2024understanding, sun2025natural, guilmeau2026convergence}. Consequently, \eqref{eq:pb_VI} is predominantly solved through Stochastic Gradient Descent (SGD) \cite{robbins1951stochastic}. However, as discussed in \cite{domke2019var, domke2023provable}, standard convergence results for SGD do not readily apply in this setting. Classical analyses typically require either uniformly bounded stochastic gradient variance \cite{nemirovski2009robust, ghadimi2013stochastic} or the more general (ABC) condition \cite{KhaledRichtarik2023}, neither of which are satisfied for modern BBVI problems \cite{alacaoglu2025towards}. To circumvent this limitation, recent literature relies on the weaker Blum-Gladyshev (BG) assumption \cite{blum1954approximation, gladyshev1965stochastic}. A stochastic gradient estimate $\mathbf{g}(\theta)$ satisfies the BG condition if $\mathbb{E}\left[\mathbf{g}(\theta)\right] = \nabla F(\theta)$ and there exist constants $(a, b) \in \mathbb{R}^+ \times \mathbb{R}^+$ and a reference point $\theta^* \in \Theta$ such that:\begin{equation}\label{eq:bg}\tag{BG}\mathbb{E}\|\mathbf{g}(\theta)\|^2 \le a\|\theta - \theta^*\|^2 + b.\end{equation} Notably, \cite{domke2023provable} proved that stochastic gradients used to solve \eqref{eq:pb_VI} satisfy this BG condition, paving the way for convergence guarantees. In this work, we go beyond the analysis in \cite{domke2023provable} to bridge the gap between the theory of SGD convergence and the practical demands of BBVI. While minibatching is standard practice in machine learning to reduce gradient variance \cite{bottou2018optimization}, its theoretical implications in BBVI have only recently been formalized for NGD in the restrctive setting of exponential families  \cite{modi2024batch, cai2024batch, guilmeau2026convergence}. Our work focuses on Projected SGD (PSGD) applied to the elliptic location-scale family of densities, a general setting that includes Gaussian, Laplace and Student-T densities. Our main contributions are the following.
\begin{itemize}[leftmargin=*]
\item Under minimal growth assumption on $\log q(x,z)$ we prove the existence of a minimizer for the NegativeELBO, a result which is usually assumed a priori rather than proven.
\item We prove convergence of Minibatch PSGD with dynamic batching and preconditioning for BBVI. We analyze convex and non-convex settings, and we provide
finite horizon and asymptotic rates.
\item To achieve the previous results, we study the projected SGD algorithm. Under the assumption that one stochastic estimate of the gradient satisfies the BG condition, we provide a convergence analysis for the minibatch estimator, highlighting the role of the minibatch size and preconditioning. 
\end{itemize}
Our results show that dynamic batching and preconditioning serve as principled and effective mechanisms for handling the unconstrained variance inherent in variational inference problems. We validate our theoretical findings through numerical experiments. A more detailed review of the literature is provided in Appendix \ref{sec:related_works}.

This paper is organized as follows: Section \ref{sec:background_BBVI} details the setting and properties of the BBVI problem, establishing the existence of a solution under minimal assumptions (Theorem \ref{thm:existence}). Section \ref{sec:app_BBVI} presents theoretical guarantees for PSGD equipped with preconditioning and dynamic batching in the context of BBVI (Theorem \ref{cor:BBVI_I}), supported by numerical illustrations. In Section \ref{sec:cvg_results}, we generalize these convergence results for functions satisfying \eqref{eq:bg} (Theorems \ref{thm:MPSGD} and \ref{thm:asymptotic}). Finally, Section \ref{sec:conclusion} discusses the limitations of this work and outlines future research directions.

\paragraph{Notation} We denote by $\mathcal{S}^d_{++} \subset \R^{d\times d}$ the space of positive definite matrices. For any matrix $\Sigma \in \R^{d\times d}$, $\sigma_{\min}$ and $\sigma_{\max}$ denote its minimum and maximum singular values, respectively. The determinant and Frobenius norm of $\Sigma$ are denoted by $|\Sigma|$ and $\|\Sigma\|_F$. For any $\theta \in \Theta \subset \R^d$ and $\Lambda \in \mathcal{S}^d_{++}$, we define $\|\theta\|_\Lambda := \sqrt{\langle\theta, \Lambda\theta\rangle}$. The projection operator onto a nonempty closed convex set $\Theta$ is denoted by $\mathcal{P}_\Theta$. The uniform sphere in $\R^d$ is denoted by $\Sd$, and its volume is given by $\VolS = \frac{\pi^{d/2}}{\Gamma(d/2 + 1)}$.

\section{Background on Black-Box Variational Inference}\label{sec:background_BBVI}

\subsection{Elliptic location-scale families}

To make the optimization of ELBO \eqref{eq:pb_VI} tractable, we restrict our variational family to \textbf{elliptic location-scale} distributions. These are parameterized by a location vector $\mu \in \R^d$ and a positive definite scaling matrix $\Sigma \in S^d_{++}$ such that:
\begin{equation}\tag{Location-Scale}\label{eq:loc-sca}
\forall x\in\R^d,\quad \search_{\mu,\Sigma}(x) = \frac{1}{|\Sigma|}\search\left(\Sigma^{-1}(x-\mu)\right),
\end{equation}
where $p$ is a base probability density function (p.d.f.) defined by a one-dimensional symmetric p.d.f. $\base$:
\begin{equation}\label{eq:base_msr}
\forall x\in\R^d,\quad \search(x)=\normC\base(\|x\|).
\end{equation}
Note that the normalization constant \(\normC\) has the following expression
\begin{equation}\label{eq:def_C}
    \normC = \left(\VolS\int_0^{\infty}r^{d-1}\base(r)dr\right)^{-1}.
\end{equation}

An important property of elliptic location-scale families is that a sample from \(\search_{\mu,\Sigma}\) can be directly obtained from one of \(\search\), i.e. if \(z\sim\search\), then \(\mu+\Sigma z\) is a sample from \(\search_{\mu,\Sigma}\).

Before investigating the properties of the resulting objective function, we recall the notion of standardized distributions below.

\begin{definition}
    The base distribution \(\search\) is said to be \textbf{standardized} if \(\E_{z\sim \search}[z_i]=0\) and \(\E_{z\sim\search}[z_i^2]=1\) for any \(i\in[d]\).
\end{definition}

In the context of elliptic location-scale families, this assumption can be guaranteed by imposing some condition on \(\base\) (see Appendix \ref{app:standardized} for detailed computations).

\begin{proposition}\label{prop:standardized}
    If \(\base:\R\to\R^+\) is a symmetric p.d.f. satisfying
    \({\int_0^\infty r^{d+1}\base(r)dr} =d {\int_0^\infty r^{d-1}\base(r)dr} ,\)
    then \(\search\) defined in \eqref{eq:base_msr} is standardized.
\end{proposition}

{\color{red}\paragraph{Discussion on exp families}}

Considering an elliptic location-scale family as search distribution allows to simplify the expression of problem \eqref{eq:pb_VI}. Indeed, note that
\[\begin{aligned}
\E_{x\sim\search_{\mu,\Sigma}}\left[\log\search_{\mu,\Sigma}(x)\right]=\int_{\R^d}\log \search_{\mu,\Sigma}(x)\search_{\mu,\Sigma}(x)dx&=\int_{\R^d}\log\left(|\Sigma|^{-1}\search(z)\right)\search(z)dz \\&=-\log|\Sigma|+\E_{x\sim\search}\left[\log\search(x)\right].\end{aligned}
\]
Since the second term is constant w.r.t. the parameters \(\mu\) and \(\Sigma\), this implies that for such search distributions, minimizing \eqref{eq:pb_VI} is equivalent to, for $f(x):=-\log q(x,z)$,
\begin{equation}\label{eq:final_objective} \tag{P}
\min_{\mu\in\mathbb{R}^d , \Sigma\in\mathcal{S}^d_{++} }\mathcal{L}(\mu, \Sigma), \quad \text{ where }\quad \mathcal{L}(\mu, \Sigma):= \E_{x \sim \search_{\mu,\Sigma}}[f(x)] -  \log |\Sigma|.
\end{equation}
\subsection{Properties of ELBO for elliptic location-scale families}
\label{sec:existence}
In the literature \cite{domke2020smooth,domke2023provable,kim2023convergence}, the existence of a solution to the problem \eqref{eq:final_objective} is generally stated as an assumption. In this section, we give sufficient conditions guaranteeing the existence of a global minimizer of \(\mathcal{L}\). Before stating the theorem, we introduce the notion of \(\growth\)-growth. 
\begin{definition}
    A function \(f:\R^d\to\R\) is said to have a \(\growth\)-growth for some \(\growth>0\) if there exist \(c_1>0\) and \(c_2\in\R\) such that for any \(x\in\R^d\),
    \(f(x)\ge c_1\|x\|^\growth-c_2.\)
\end{definition}
\begin{theorem}\label{thm:existence}
    Let \(f:\R^d\to \R\) be a continuous function and assume that the base distribution \(\search:\R^d\to\R\) defined by \eqref{eq:base_msr} is standardized. Suppose that one of the two following conditions holds:
    \begin{itemize}
        \item \(f\) has a \(2\)-growth (i.e. a quadratic growth),
        \item \(f\) has a \(\growth\)-growth for some \(\growth>0\) and \(\search\) is log-concave.
    \end{itemize}
    Then, the function \(\mathcal{L}\) defined in \eqref{eq:final_objective}
    has at least one global minimizer on \(\R^d\times\mathcal{S}^d_{++}\).
\end{theorem}

The proof of Theorem \ref{thm:existence} is deferred to Appendix \ref{proof:existence}. 

Given the above existence theorem, one can combine it to \cite[Theorem~7]{domke2020smooth} that relates the solution of \eqref{eq:final_objective} to the smoothness of \(f\).

\begin{corollary}\label{cor:existence}
    Let \(f:\R^d\to\R\) be a continuous \(L\)-smooth function for some \(L>0\) and let \(\search\) defined in \eqref{eq:base_msr} be standardized. If \(f\) and \(\search\) satisfy one of the assumptions of Theorem \ref{thm:existence}, then there exists a global minimizer of \eqref{eq:final_objective} and any minimizer belongs to \begin{equation}\label{eq:Theta_L}\Theta_L=\left\{(\mu,\Sigma)\in\R^d\times\mathcal{S}^d_{++},~\sigma_{min}(\Sigma)\ge\frac{1}{\sqrt{L}}\right\}.\end{equation}
\end{corollary}

The above result ensures that if the function \(f\) is \(L\)-smooth, the eigenvalues of any optimal scaling matrix \(\Sigma^*\) of \eqref{eq:final_objective} are a priori bounded away from 0. This property can be exploited during optimization by minimizing \eqref{eq:final_objective} directly on the closed and convex set \(\Theta_L\) which can be done e.g. applying Projected Stochastic Gradient Descent \cite{domke2023provable}.

We summarize below additional properties of problem \eqref{eq:final_objective}.
\begin{proposition}\label{prop:properties}
    Let \(f:\R^d\to\R\) and \(\search:\R^d\to\R\) be defined by \eqref{eq:base_msr} according to a symmetric base distribution \(\base\).
    \begin{enumerate}[leftmargin=*]
        \item \textbf{Convexity} \cite[Proposition~1]{titsias2014doubly}: if \(f\) is convex and \(\search\) is log-concave, then \(\mathcal{L}\) is jointly convex in \((\mu,\Sigma)\).
        \item \textbf{Smoothness} \cite[Theorem~11]{domke2020smooth}: if \(f\) is \(L\)-smooth for some \(L>0\) and \(\search\) is standardized, then \(\mathcal{L}\) is \(2L\)-smooth on \(\Theta_L\).
        \item \textbf{Strong convexity} \cite[Theorem~9]{domke2020smooth}: if \(f\) is \(m\)-strongly convex and \(\search\) is log-concave and standardized, then \(\mathcal{L}\) is jointly \(m\)-strongly convex in \((\mu,\Sigma)\).
    \end{enumerate}
\end{proposition}

\subsection{Gradient estimation} 
\label{sec:derivatives}

Since we consider gradient based methods in this work, an important step lies in estimating the gradient of \(\mathcal{L}\). We start by treating the first term of the objective that we denote as \(\data(\mu,\Sigma):=\E_{x\sim\search_{\mu,\Sigma}}\left[f(x)\right]\). We will consider the \emph{reparameterization} estimate \cite{kingma2014efficient} of \(\nabla \data\) which leverages the following formulation thanks to the change of variables $x=\mu+\Sigma z$
\[
\data(\mu,\Sigma)=\E_{z\sim\search}\left[f(\mu+\Sigma z)\right].
\]
Based on the above rewriting of \(\data\), one can directly show that
\[\left\{\begin{aligned}\nabla_\mu\data(\mu,\Sigma)&=\E_{z\sim\search}\left[\nabla f(\mu+\Sigma z)\right]\\
\nabla_\Sigma\data(\mu,\Sigma)&=\E_{z\sim\search}\left[\nabla f(\mu+\Sigma z)z^T\right].\end{aligned}\right.\]
The gradient of \(\data\) will be estimated by sampling \(z\) from \(\search\) and defining
\begin{equation}
\gdatamu\left(z;\mu,\Sigma\right)=\nabla f(\mu+\Sigma z),\quad\gdataS\left(z;\mu,\Sigma\right)=\nabla f(\mu+\Sigma z)z^T.
\end{equation}
The above estimates of \(\nabla\data\) are unbiased, and by supposing that \(f\) is \(L\)-smooth, \cite[Theorem~3]{domke2019var} gives us the following bounds on their variance, paving the way to show that the \eqref{eq:bg} condition is satisfied for these estimators. 
\begin{theorem}\label{thm:g_variance}
    Let \(f\) be \(L\)-smooth function and \(\bar z\in\R^d\) be a stationary point of \(f\). If \(\search\) is standardized, then
    \[\E_{z\sim\search}\left[\|\gdatamu(z;\mu,\Sigma)\|^2\right]\le L^2\|\mu-\bar z\|^2+L^2\|\Sigma\|^2_F,\]
    and
    \[\E_{z\sim\search}\left[\|\gdataS(z;\mu,\Sigma)\|_F^2\right]\le L^2d\|\mu-\bar z\|^2+L^2(d-1+\kappa(\search))\|\Sigma\|^2_F,\]
    where \(\kappa(\search)=\E_{z\sim\search}\left[z_i^4\right]\) is the kurtosis.
\end{theorem}

\begin{remark}
    For a distribution \(\search\) defined via a base distribution \(\base\) as in \eqref{eq:base_msr}, the kurtosis can be directly written according to \(\base\):
    \[\kappa(\search)=\frac{3}{d(d+2)}\frac{\int_{0}^\infty r^{d+3}\base(r)dr}{\int_{\R^d}r^{d-1}\base(r)dr}.\]
\end{remark}

Since the objective function \(\mathcal{L}\) also contains a second term \(-\log|\Sigma|\), whose gradient is explicitly computable, the gradient of \(\mathcal{L}\) can be estimated from a sample \(z\sim\search\) computing \(\gLmu(z;\mu,\Sigma)\) equal to \(\gdatamu(z;\mu,\Sigma)\) and
\[
\gLS(z;\mu,\Sigma)=\gdataS(z;\mu,\Sigma)-\Sigma^{-1}.
\]
The variance of this gradient estimate can be upper bounded by applying \cite[Theorem~3]{domke2023provable}.
\begin{theorem}\label{thm:g_Sigma_ent}
    Let \(\bar\sigma>0\). Let \(f\) be a \(L\)-smooth function and \(\bar z\in\R^d\) be a stationary point of \(f\). If \(\search\) is standardized, then for any \(\Sigma\in\mathcal{S}^d_{++}\) such that \(\sigma_{min}(\Sigma)\ge\bar\sigma^{-\frac{1}{2}}\),
    \[ \E_{z\sim\search}\left[\|g_\Sigma(z;\mu,\Sigma)\|_F^2\right]\le 2L^2d\|\mu-\bar z\|^2+2L^2(d-1+\kappa(\search))\|\Sigma\|^2_F+2 d\bar\sigma.\]
\end{theorem}

This work aims at leveraging minibatches for estimating the gradient of \(\mathcal{L}\) more accurately. We will therefore consider sampling multiple vectors \(z_i\sim\search\) to approximate \(\nabla \mathcal{L}\) with
\begin{equation}\label{eq:def_MC}\gLmuS^N\left(\left(z_i\right)_{i=1}^N;\mu,\Sigma\right)=\frac{1}{N}\sum_{i=1}^N\gLmuS\left(z_i;\mu,\Sigma\right).\end{equation}
For the sake of clarity, the above estimates will be simply denoted \(\gLmuS^N\) or \(\gLmuS^N(\mu,\Sigma)\).
\begin{remark}
    Note that the considered estimators of  \(\nabla \mathcal{L}\) require \(\nabla f\) to exist and be available. Another approach to avoid using first order information on \(f\) (and ultimately \(q\)) is to exploit the score-based expression of the gradient \cite{glynn1990likelihood,rubinstein1989sensitivity}
\[
\nabla_{\mu,\Sigma}\data(\mu,\Sigma)=\E_{z\sim\search_{\mu,\Sigma}}\left[f(x)\nabla_{\mu,\Sigma}\log\search_{\mu,\Sigma}(x)\right].
\]
While this strategy is relevant and can be related to several algorithms beyond Variational Inference (e.g. in Reinforcement Learning with REINFORCE gradient estimate \cite{williams1992simple} or in Natural Evolutionary Strategies \cite{wierstra2014natural}), the analysis of methods based on these gradient estimates is beyond the scope of this paper and could be the object of further developments.
\end{remark}
\section{Guarantees for Black-Box Variational Inference}\label{sec:app_BBVI}

In this section, we state the main convergence results of the paper. These are an instantiation of the general ones for preconditioned minibatch PSGD, that we defer to Section \ref{sec:cvg_results}.  We  consider the assumptions of  Corollary \ref{cor:existence}. Then \eqref{eq:final_objective} has a solution that belongs to the set \(\Theta_L\) defined in \eqref{eq:Theta_L}.
One can therefore apply the following algorithm, given \(\left(\mu^0,\Sigma^0\right)\in\Theta_L\),
\begin{equation}\label{eq:MPSGDVI}
    \forall k\in\mathbb{N},\quad\left(\mu^{k+1},\Sigma^{k+1}\right)=\mathcal{P}_{\Theta_L}\left(\left(\mu^{k},\Sigma^{k}\right)-\gamma_k\Lambda\gLmuS^{N_k}\left(\mu^{k},\Sigma^{k}\right)\right),
\end{equation}
where \(\gLmuS^{N}\) is defined in \eqref{eq:def_MC} and \begin{equation}\label{eq:Lambda_VI}\Lambda=\begin{pmatrix}
    \lambda_\mu I_d &0\\0&\lambda_\Sigma I_{{d^2}}
\end{pmatrix},\quad \left(\lambda_\mu,\lambda_\Sigma\right)\in(0,1]^2.\end{equation}
Note that, at each iteration, applying \(\mathcal{P}_{\Theta_L}\) to \(\left(\mu,\Sigma\right)\) requires computing a Singular Value Decomposition (SVD) of \(\Sigma\), which has a complexity of order \(d^3\). Given a budget of stochastic gradient evaluations, using mini-batches to compute a more accurate gradient reduces the number of required iterations (and therefore of SVD computations). The following results state that this strategy does not affect the convergence properties of the proposed algorithm in terms of total number of stochastic gradient evaluations. We start with the case $N_k=N$ constant in the finite horizon scenario, where the results are more readable. Then we extend to the case of dynamic mini-batching in the asymptotic regime.
\begin{theorem}\label{cor:BBVI_I}
    Let \(f:\R^d\to\R\) be convex, \(L\)-smooth and have a \(\growth\)-growth for some \(L>0\) and \(r>0\). Let \(\search:\R^d\to\R\) be defined as in \eqref{eq:base_msr}, standardized and log-concave. Let \(\left(\mu^k,\Sigma^k\right)_{k\in\mathbb{N}}\) be the sequence generated by \eqref{eq:MPSGDVI} for \(\gamma_k=\frac{1}{2L}\) and \(N_k=N\) for any \(k\in\mathbb{N}\), and \(\Lambda\) defined by \eqref{eq:Lambda_VI}. Given a budget of \(E\) stochastic gradient evaluations, we have that
    \begin{enumerate}[leftmargin=*]
        \item If \(\lambda_\mu=\lambda_\Sigma=1\) and \(N=\sqrt{d+\kappa(\search)}\sqrt{E}\), then after \(K=\sqrt{E}/\sqrt{d+\kappa(\search)}\) iterations,
        \[\begin{aligned}
            \E\left[\mathcal{L}\left(\bar{\mu}^{K},\bar{\Sigma}^{K}\right)\right]-\mathcal{L}\left(\mu^*,\Sigma^*\right)\le   \frac{L\sqrt{d+\kappa(\search)}}{\sqrt{E}}&\left( 2\left\| \mu^{0}-\mu^{*}\right\|^{2} + 2\left\| \Sigma^{0}-\Sigma^{*}\right\|_F^{2}\right.\\&~\left. + \left\|\mu^*-\tilde\mu\right\|^2+\left\|\Sigma^*\right\|^2_F+\frac{1}{2L}\right).\end{aligned}\]
        \item If \(\lambda_\mu=1\) and \(\lambda_\Sigma=\frac{1}{2(d-1+\kappa(\search))}\) and \(N=\sqrt{E}\), then after \(K=\sqrt{E}\) iterations,
        \[\begin{aligned}
            \E\left[\mathcal{L}\left(\bar{\mu}^{K},\bar{\Sigma}^{K}\right)\right]-\mathcal{L}\left(\mu^*,\Sigma^*\right)\le   \frac{L}{\sqrt{E}}&\left( 2\left\| \mu^{0}-\mu^{*}\right\|^{2} + 2(d-1+\kappa(\search))\left\| \Sigma^{0}-\Sigma^{*}\right\|_F^{2}\right.\\&~\left. + \left\|\mu^*-\tilde\mu\right\|^2+\left\|\Sigma^*\right\|^2_F+\frac{1}{4L}\right),\end{aligned}\]
    \end{enumerate} 
    where \(\left(\mu^*,\Sigma^*\right)\) is a global minimizer of \(\mathcal{L}\) on \(\R^d\times\mathcal{S}^d_{++}\), \(\tilde \mu\) is a minimizer of \(f\), and \(\left(\bar \mu^k\right)_{k\in\mathbb{N}}\) and \(\left(\bar \Sigma^k\right)_{k\in\mathbb{N}}\) are defined as
    \[\bar{\mu}^{k}=\frac{\sum_{i=0}^{k-1}\tau^{i+1}\mu^i}{\sum_{i=0}^{k-1}\tau^{i+1}},\quad\bar{\Sigma}^{k}=\frac{\sum_{i=0}^{k-1}\tau^{i+1}\Sigma^i}{\sum_{i=0}^{k-1}\tau^{i+1}},\quad \tau=\frac{K}{K+1}.\]
\end{theorem}

Theorem \ref{cor:BBVI_I} is an application of Theorem \ref{thm:MPSGD} to BBVI. The proof, given in Appendix \ref{app:cor_VI}, is based on showing that in the BBVI setting the stochastic estimators of the gradient introduced in Section \ref{sec:derivatives} satisfy the \ref{eq:bg} condition.

In the case of dynamic mini-batching, with the same proof, we apply also Theorem \ref{thm:asymptotic}. Then, under the summability condition $\sum_{i=1}^{+\infty} \gamma_i^2/N_i < +\infty$, we get for some constant $C>0$ and a suitable averaging (see Theorem~\ref{thm:asymptotic})
\begin{equation}
    \E\left[\mathcal{L}\left(\bar{\mu}^{K},\bar{\Sigma}^{K}\right)\right]-\mathcal{L}\left(\mu^*,\Sigma^*\right) \leq  \frac{C}{\sum_{k=1}^K \gamma_k}
\end{equation}
and
\begin{equation}
\min_{k\in[K]}\left(\mathcal{L}\left(\bar{\mu}^{K},\bar{\Sigma}^{K}\right)-\mathcal{L}\left(\mu^*,\Sigma^*\right)\right) =\mathcal{O}\left(\left(\sum_{k=1}^K \gamma_k\right)^{-1}\right) \ a.s.
\end{equation}
Then, for different policies on the step size and the batch size, we can express these rates in terms of the number of gradient evaluations \(E\). In this setting, the total number of gradient evaluations after \(K\) iterations is \(E=\sum_{k=1}^K N_k\). Then, if we denote \(a_k \asymp b_k\) to indicate that the sequences \((a_k)\) and \((b_k)\) are asymptotically of the same order, one can adopt one of the following strategies (see Appendix \ref{app:asymptotic_rates} for detailed computations):
\begin{itemize}[leftmargin=*]
    \item \textbf{Decaying step size and constant batch size:} Set \(\gamma_k\asymp k^{-\frac{1}{2}-\varepsilon}\) and \(N_k\asymp1\) for \(\varepsilon>0\). Then, the summability condition is satisfied and the rates are \(\mathcal{O}\left(E^{-\frac{1}{2}+\varepsilon}\right)\).
    \item \textbf{Decaying step size and increasing batch size:} Set \(\gamma_k\asymp k^{-\frac{1}{2}}\) and \(N_k\asymp k^\sigma\) for \(\sigma>0\). Then, the summability condition is satisfied and the rates are \(\mathcal{O}\left(E^{-\frac{1}{2(\sigma+1)}}\right)\).
    \item \textbf{Constant step size and increasing batch size:} Set \(\gamma \asymp 1\) and \(N_k\asymp k^{1+\nu}\) for \(\nu>0\). Then, the summability condition is satisfied and the rates are \(\mathcal{O}\left(E^{-\frac{1}{2+\nu}}\right)\).
\end{itemize}
It is important to notice that by adjusting respectively \(\varepsilon\), \(\sigma\) and \(\nu\) (i.e. making them tend to \(0\)), each strategy allows to reach a rate arbitrarily close to \(\mathcal{O}\left(\frac{1}{\sqrt{E}}\right)\). The best strategy therefore depends on the context of application: if the memory available is small making large batch processing infeasible, the first regime is relevant; if the extra-cost per iteration is high or the gradient computations can be parallelized, then the third setting is more appropriate.

\begin{remark}
Theorem \ref{cor:BBVI_I} gives finite time rates for two different approaches, either setting a common step size for \(\mu^k\) and \(\Sigma^k\), either balancing it to take in account the larger variance induced by the gradient estimate in \(\Sigma\). The resulting rates show that, in the first case, the weight of the dimension is evenly spread between the left hand side terms with a factor \(\sqrt{d+\kappa(\search)}\). In the second case, it only appears in front of the term \(\left\|\Sigma^0-\Sigma^*\right\|_F^2\), with a factor \(d-1+\kappa(\search)\). Note that in practice, setting \(\Sigma^0\) as \(\frac{1}{\sqrt{L}}I_{d}\) is an efficient heuristic that allows limiting this term.\end{remark}

\begin{remark}
    Other approaches to reduce the burden of SVD computations are based on the restriction of \(\Sigma\) to the space of lower triangular matrices \cite{domke2023provable,kim2023convergence}. 
\end{remark}

\begin{figure}[h]
    \centering
    \includegraphics[width=0.5\linewidth]{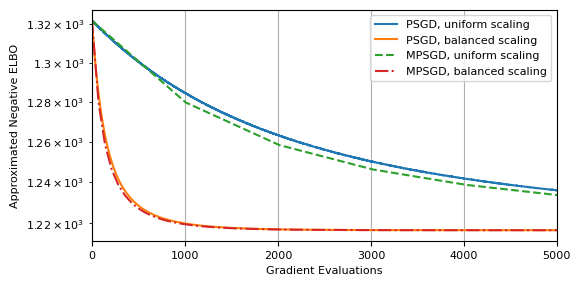}\hfill
    \includegraphics[width=0.5\linewidth]{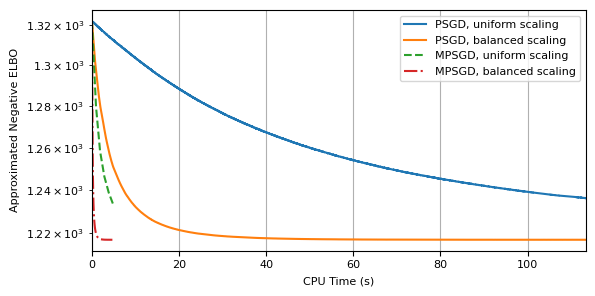}
    \caption{Evolution of the Negative ELBO (approximated with \(100\) samples) along the iterates of Projected Stochastic Gradient Descent (PSGD) \cite{domke2023provable} (in blue), PSGD with scaled step size (\(\Lambda\) as in the second point of Theorem \ref{cor:BBVI_I}, in orange), MPSGD with \(\Lambda=I_{d+d^2}\) (in green, dashed line) and MPSGD with adapted \(\Lambda\) (see Theorem \ref{cor:BBVI_I}, in red), w.r.t. the number of gradient evaluations (left) and the CPU time (right). We consider a Bayesian Logistic Regression problem with a Laplace search distribution, for \(d=200\). See more details in Appendix \ref{app:experiment_details}.}
    \label{fig:comp1}
\end{figure}

Figure \ref{fig:comp1} illustrates the benefit of adapted scaling and minibatching for a high dimension problem (\(d=200\)): by defining a preconditioning matrix \(\Lambda\) adapted to the structure of the problem, as in Theorem \ref{cor:BBVI_I}, one can observe a significant improvement for both PSGD and MPSGD. In addition, the use of minibatch is particularly relevant for the general computation time, since SVD decompositions when \(d=200\) are costly.

Under the same set of assumptions on \(f\) and \(\search\), asymptotic guarantees can be derived from Theorem \ref{thm:asymptotic}, allowing for variable step size and dynamic batching to reach an almost sure convergence rate arbitrarily close to \(\mathcal{O}\left(\frac{1}{\sqrt{E}} \right)\).

In addition, for a non log-concave target distribution \(q(\cdot,z)\), Theorem \ref{thm:non_convex1} provides a rate under the assumption that the sequence \(\left(\mu^k,\Sigma^k\right)_{k\in\mathbb{N}}\) is bounded. Such a condition can be ensured for instance by constraining directly the search space to a compact set \(\bar\Theta\subset\Theta_L\), potentially excluding the true solution of \eqref{eq:final_objective}.

\section{Convergence of Minibatch Projected Stochastic Gradient Descent under Assumption \eqref{eq:bg}}\label{sec:cvg_results}

In this section we consider the more general setting of the minimization of a differentiable objective function \(F:\R^D\to \R\) satisfying Assumption \eqref{eq:bg} over the nonempty closed convex set \(\Theta\subset\R^D\).  We give several convergence results for Minibatch Projected Stochastic Gradient Descent \eqref{eq:MPSGD}: for some \(\theta^0\in\R^D\),
\begin{equation}\tag{MPSGD}\label{eq:MPSGD}
    \forall k\in\mathbb{N},\quad \theta^{k+1}=\mathcal{P}_{\Theta}\left(\theta^k-\gamma_k\Lambda{\mathbf{g}}^{N_k}\left(\theta^k\right)\right),
\end{equation}
where \(\left(\gamma_k\right)_{k\in\mathbb{N}}\) is the sequence of step sizes, \(\left(N_k\right)_{k\in\mathbb{N}}\) is the sequence of minibatch size per iteration and \(\Lambda\in\R^{D\times D}\) is a diagonal preconditioning (when equal to \(I_D\), one recovers the classical formulation). The gradient \(\nabla F\left(\theta^k\right)\) is approximated with \(\mathbf{g}^{N_k}\left(\theta^k\right)\) which is an average of \(N_k\) unbiased estimates \({g}_i(\theta^k)\) of \(\nabla F\left(\theta^k\right)\):
\begin{equation}\label{eq:MCg}
   {\mathbf{g}}^{N}\left(\theta\right)=\frac{1}{N}\sum_{i=1}^{N}{g}_i(\theta).
\end{equation}

The preconditioning matrix \(\Lambda\) acts as a simple scaling onto the step size, allowing us to adapt it to each coefficient. In particular, if the variance of the gradient estimate is larger for some coefficients, adapting \(\Lambda\) can reduce its impact. We will suppose that \(\Lambda\) and \(\Theta\) satisfy the following assumption. Note that for \(\Lambda=I_D\), no separability assumption on \(\Theta\) is necessary.
\begin{assumption}
\label{ass:block_separable_metric}
Assume the parameter space \(\R^D\) decomposes into \(m\) orthogonal subspaces, \(\R^D = \mathcal{X}_1 \times \dots \times \mathcal{X}_m\), such that:
\begin{enumerate}[leftmargin=*]
    \item \textbf{Set Separability:} The closed convex constraint set \(\Theta \subseteq \mathcal{X}\) is a Cartesian product \(\Theta = \Theta_1 \times \dots \times \Theta_m\), where each \(\Theta_i \subseteq \mathcal{X}_i\) is closed and convex.
    \item \textbf{Block-Isotropic Metric:} The matrix \(\Lambda\) acts as a scalar multiple of the identity on each subspace. Specifically, \(\Lambda = \operatorname{diag}(\lambda_1 I_1, \dots, \lambda_m I_m)\), where \(\lambda_i \in (0, 1]\) and \(I_i\) is the matrix associated to the identity of  \(\mathcal{X}_i\).
\end{enumerate}
\end{assumption}

Incorporating the metric \(\Lambda\) into the method requires us to adapt the assumption \eqref{eq:bg}. 

\begin{definition}
    A gradient estimate \({g}\) satisfies \eqref{eq:bgLambda} with parameters \((a,b,\theta^*)\) if for any \(\theta\in\Theta\), \(\E\left[{g}(\theta)\right]=\nabla F(\theta)\) and 
    \begin{equation}\tag{BG$_\Lambda$}\label{eq:bgLambda}
        \E\left\|{g}(\theta)\right\|_{\Lambda}^2\le a\|\theta-\theta^*\|^2_{\Lambda^{-1}}+b.
    \end{equation}
\end{definition}

When the parameter vector \(\theta\) decomposes into \(m\) block components \(\theta_i\), and the gradient estimate with respect to each component satisfies the standard bound \eqref{eq:bg}, we can explicitly construct the constants \(a\) and \(b\) such that the full gradient estimate satisfies \eqref{eq:bgLambda}, see Proposition~\ref{prop:3} in Appendix~\ref{app:standardized}.

We give in Theorem \ref{thm:MPSGD} and Theorem \ref{thm:asymptotic} theoretical guarantees for \eqref{eq:MPSGD} considering a convex function \(F:\R^D\to\R\), respectively for a finite time horizon and asymptotically. A nonconvex analysis is provided in Appendix \ref{app:nonconvex}, under the assumption that the sequence \(\left(\theta^k\right)_{k\in\mathbb{N}}\) is bounded. Although restrictive, this assumption allows using classical tools from stochastic optimization \cite{ghadimi2016mini}.

\subsection{Finite time behavior}
We start by giving finite time guarantees for \eqref{eq:MPSGD} under assumption \eqref{eq:bgLambda} for a convex function \(F\). The step size \(\gamma\) and the batch size \(N\) are chosen constant along iterations. The following theorem can be seen as a generalization of \cite[Theorem~11]{domke2023provable} and its proof is given in Appendix \ref{proof:MBPSGD}.

\begin{theorem}\label{thm:MPSGD}
    Let $F:\R^D\to\R$ be a convex function, differentiable and \(L\)-smooth on a nonempty closed convex set $\Theta$. Suppose that \(F\) has a global minimizer on \(\Theta\) and denote it \(\theta^*\). We assume that \(\Theta\) and \(\Lambda\) satisfy Assumption \ref{ass:block_separable_metric} and that, for some \((a,b)\), each gradient estimate \(g_i\) in \eqref{eq:MCg} satisfies \eqref{eq:bgLambda} with parameters \((a,b,\theta^*)\).
    Then, for any \(K\ge0\), the sequence \(\left(\theta^k\right)_{k=0}^K\) generated by \eqref{eq:MPSGD} with \(\gamma_k=\gamma\in\left(0,\frac{1}{L}\right]\) and \(N_k=N\in\mathbb{N}^*\) is such that
    \[
    \mathbb{E}\left[F(\bar{\theta}^{K})\right]-F\left(\theta^*\right)\leq\frac{\gamma}{2N}\left(a\norm{\theta^{0}-\theta^{*}}_{\Lambda^{-1}}^{2}+b\right)+\frac{1}{2\gamma K}\left\|\theta^0-\theta^*\right\|_{\Lambda^{-1}}^2,
    \]
    where $\tau=\frac{1}{1+aN^{-1}\gamma^{2}}$
    and $\bar{\theta}^{K}=\frac{\sum_{k=0}^{K-1}\tau^{k+1}\theta^{k}}{\sum_{k=0}^{K-1}\tau^{k+1}}$. 

\end{theorem}

\begin{remark}[Optimizing the bound]\label{rmk:optimize_bound}
For a given number of gradient evaluations \(E\), $E=N\cdot K$ and so, denoting \(\alpha:=\sqrt{{N}/{K}}\), we have \(N=\alpha\sqrt{E}\) and \(K=\frac{\sqrt{E}}{\alpha}\). Then, if \(D:=\norm{\theta^{0}-\theta^{*}}_{\Lambda^{-1}}^{2}\), the right hand side of the upper bound given in Theorem \ref{thm:MPSGD} can be rewritten as
\(\frac{\gamma\left(aD+b\right)}{2\alpha\sqrt{E}}+\frac{\alpha D}{2\gamma\sqrt{E}}=:\Psi(\alpha).\)
This function is minimized for \(\alpha^*=\gamma\sqrt{a+\frac{b}{D}}\) with value
\(\Psi(\alpha^*)=\frac{\sqrt{aD^2+bD}}{\sqrt{E}}\), leading to the bound
\[
    \mathbb{E}\left[F(\bar{\theta}^{K})\right]-F\left(\theta^*\right)\leq\frac{1}{\sqrt{E}}\left\|\theta^0-\theta^*\right\|_{\Lambda^{-1}}\sqrt{a\norm{\theta^{0}-\theta^*}_{\Lambda^{-1}}^{2}+b}.
    \]
This suggests that the step size \(\gamma\) and the amount of samples per iteration \(N\) to compute the gradient should be chosen proportionally to get the best guarantees. In addition, the constraint on the step size (i.e. \(\gamma\in\left(0,\frac{1}{L}\right]\)) also acts as a constraint on the maximal batch size, if aiming at the best bound.
\end{remark}


\subsection{Asymptotic analysis with dynamic batching}
\label{sec:asymptotic}

In this section, we investigate the asymptotic behavior of \eqref{eq:MPSGD} and distinguish several strategies for defining the step size sequence \(\left(\gamma_k\right)_{k\in\mathbb{N}}\) and the batch size \(\left(N_k\right)_{k\in\mathbb{N}}\). The proof of the theorem below is deferred to Appendix \ref{app:asymptotic}.

\begin{theorem}\label{thm:asymptotic}
    Let $F:\R^D\to\R$ be a convex function, differentiable and \(L\)-smooth on a nonempty closed convex set $\Theta$. Suppose that \(F\) has a global minimizer on \(\Theta\) and denote it \(\theta^*\). We assume that \(\Theta\) and \(\Lambda\) satisfy Assumption \ref{ass:block_separable_metric} and that, for some for some \((a,b)\), each gradient estimate \({g_i}\) in \eqref{eq:MCg} satisfies \eqref{eq:bgLambda} with parameters \((a,b,\theta^*)\).
    Let the sequence \(\left(\theta^k\right)_{k\in\mathbb{N}}\) be generated by \eqref{eq:MPSGD} with \(\gamma_k\in\left(0,\frac{1}{L}\right]\) and \(N_k\in\mathbb{N}^*\) for any \(k\in\mathbb{N}\). 
   Define $\alpha_{k}=\prod_{i=0}^{k-1} (1+a\gamma_i^2/N_i)^{-1}$ and $\bar{\theta}^{K}=\left(\sum_{k=0}^{K-1}\gamma_k\alpha_{k+1}\theta^{k}\right)/\sum_{k=0}^{K-1}\gamma_k\alpha_{k+1}$. Then
    \begin{align*}
 \mathbb{E}\left[ F(\bar{\theta}^{K}) - F(\theta^{*}) \right] & \leq \frac{1}{2 \sum_{k=0}^{K-1} \gamma_k\left(\exp\left(-a\sum_{i=0}^{k}\frac{\gamma_i^2}{N_i}\right)\right)} \left( \left\| \theta^{0}-\theta^{*} \right\|_{\Lambda^{-1}}^{2} + \frac{b}{a}\right). 
 \end{align*}
    If, in addition, it holds that \(\sum_{k\in\mathbb{N}}\frac{\gamma_k^2}{N_k}<\infty\), then
    \(\sum_{k\in\mathbb{N}}\gamma_k\left( F\left(\theta^k\right) - F(\theta^{*}) \right)<\infty,\quad a.s.\)
\end{theorem}
First note that according to Theorem \ref{thm:asymptotic}, 
\begin{align*}
 \mathbb{E}\left[ F(\bar{\theta}^{K}) - F(\theta^{*}) \right] & \leq \frac{\exp\left(a\sum_{i=0}^{K-1}\frac{\gamma_i^2}{N_i}\right)}{2 \sum_{k=0}^{K-1} \gamma_k} \left( \left\| \theta^{0}-\theta^{*} \right\|_{\Lambda^{-1}}^{2} + \frac{b}{a}\right). 
 \end{align*}
In particular, if \(\sum_{k\in\mathbb{N}}\frac{\gamma_k^2}{N_k}<\infty\), then the numerator is bounded. Moreover, from the second statement of the theorem, we have that
\[\min_{k\in[K]}\left(F\left(\theta^k\right)-F\left(\theta^*\right)\right)=\mathcal{O}\left(\left(\sum_{k=1}^K \gamma_k\right)^{-1}\right)\quad a.s.\]

\section{Conclusion and Perspectives}\label{sec:conclusion}
In this work, we bridge the gap between stochastic optimization theory and the practical challenges of Black-Box Variational Inference by analyzing the convergence of a preconditioned minibatch projected SGD under the relaxed Blum-Gladyshev (BG) variance condition. Focusing on elliptic location-scale families, we first establish the existence of a minimizer for the negative ELBO, which is a fundamental property frequently assumed a priori in the literature. Furthermore, we provide comprehensive convergence guarantees for finite-time and asymptotic regimes, demonstrating that dynamic batching and preconditioning mitigate the unbounded variance inherent in BBVI.

Despite these theoretical advances, we recognize several limitations in our current analysis. First, while we focus on preconditioned projected SGD, the algorithmic landscape for BBVI is broad, and our framework does not presently encompass other prominent optimization strategies, such as Natural Gradient Descent (NGD). Second, our main results rely on the convexity of \(f\) (or \(F\)), while our convergence guarantees in the nonconvex regime (detailed in Appendix \ref{app:nonconvex}) rely on a bounded sequence assumption, which remains a restrictive condition. Finally, bridging our theoretical bounds with practical implementations is complicated by the fact that crucial problem-specific parameters, such as the Lipschitz constant \(L\), are often unknown or difficult to accurately estimate in complex VI scenarios.

Moving forward, a natural extension of this framework involves analyzing natural gradient and score-based gradient estimators under similar unbounded variance assumptions. Furthermore, adapting anchoring techniques like Halpern iteration \cite{Halpern1967FixedPO} to the VI setting could be a promising avenue, since the latter has a favorable behavior under the BG assumption, even in nonconvex settings \cite{alacaoglu2025towards, fazla2026lower}.

\begin{ack}
H. L. and L. R. acknowledge the financial support of the Ministry of Education, University and Research (FARE grant ML4IP R205T7J2KP). C. M., S. V. and L. R. acknowledge the support of the US Air Force Office of Scientific Research (FA8655-22-1-7034). The research by C. M. and S. V. has been supported by the MUR Excellence Department Project awarded to Dipartimento di Matematica, Universita di Genova, CUP D33C23001110001. S. V. is a member of the Gruppo Nazionale per l’Analisi Matematica, la Probabilità e le loro Applicazioni (GNAMPA) of the Istituto Nazionale di Alta Matematica (INdAM). L. R. acknowledges the financial support of the European Commission (Horizon Europe grant ELIAS 101120237). L. R. is affiliated to the Istituto Italiano di Tecnologia (IIT). This work represents only the view of the authors. The European Commission and the other organizations are not responsible for any use that may be made of the information it contains.									
\end{ack}

\bibliographystyle{abbrv}
\bibliography{ref}


\appendix

\section{Related works}\label{sec:related_works}

\paragraph{Black-Box Variational Inference (BBVI).} 
BBVI provides a scalable framework for posterior approximation by relying solely on stochastic gradients of the ELBO \cite{kingma2014efficient, ranganath2014black}. A substantial body of subsequent work has developed algorithmic improvements to stabilize these estimates, primarily through variance reduction techniques, smoothing mechanisms, and boosting strategies \cite{archer2015black, locatello2018boosting, domke2019var, domke2020smooth, modi2023variational, burroni2023u, giordano2024black}. 
While the empirical success of BBVI is well-established, developing robust theoretical guarantees is an ongoing challenge due to the unbounded variance of the gradient estimates. Addressing this, \cite{domke2023provable} and \cite{kim2023convergence} establish critical convergence guarantees for BBVI using reparameterization gradients under relaxed assumptions. Highly relevant to our setting, \cite{hotti2024benefits} bounds the variance of gradient estimates when reparameterizing the scaling parameter of elliptic location-scale families. Our analysis directly builds upon these foundations by extending convergence guarantees under the Blum-Gladyshev condition to broader scenarios.

\paragraph{Minibatching in BBVI.} 
While utilizing minibatches is standard practice in classical stochastic optimization to reduce gradient variance and accelerate convergence \cite{bottou2018optimization}, its theoretical implications specifically within BBVI have only recently been formalized \cite{modi2024batch, cai2024batch, guilmeau2026convergence}. Inline with classical optimization theory, our work provides an intuitive and rigorous understanding of minibatch effects in the VI setting. 

\paragraph{Natural Gradient Descent for VI.} 
A prominent alternative to standard SGD in Variational Inference is Natural Gradient Descent (NGD) \cite{amari1998natural}, which leverages the information geometry of the variational family. Recent literature has extensively analyzed NGD when the search distributions belong to exponential families \cite{wu2024understanding, sun2025natural, guilmeau2026convergence}. For instance, \cite{wu2024understanding} establishes convergence rates for Stochastic NGD (SNGD) under conjugate models (which inherently implies a strongly convex data term), while \cite{sun2025natural} extends these rates to non-conjugate settings. Closest to our theoretical goals, \cite{guilmeau2026convergence} provides an analysis of SNGD in VI but restricts their scope to exponential families. Additionally, \cite{zantedeschi2026fisher} studies the oracle complexity of Natural Gradients with minibatches in a broader optimization context. In contrast to these approaches, our work focuses on Projected SGD applied to elliptic location-scale families, providing convergence guarantees without relying on the restrictive geometric properties or conjugacy requirements of exponential families.

\paragraph{Stochastic Gradient Descent and Variance Assumptions.}
Stochastic Gradient Descent (SGD) \cite{robbins1951stochastic} remains the foundational workhorse for optimization. Classical convergence analyses of SGD typically rely on the assumption of uniformly bounded variance \cite{nemirovski2009robust, ghadimi2013stochastic, bottou2018optimization}, or more recently, on the Expected Smoothness (ABC) condition \cite{KhaledRichtarik2023}. However, as highlighted by \cite{alacaoglu2025towards}, the bounded variance assumption is often overly restrictive for modern machine learning problems, failing to hold even in elementary unconstrained settings. 

To circumvent this limitation, recent literature has revitalized the Blum-Gladyshev (BG) condition \cite{blum1954approximation, gladyshev1965stochastic}, which relaxes the variance bound by allowing it to grow quadratically with the distance to the optimum. This condition appears in various recent contexts \cite{wang2016stochastic, cui2021analysis, jacobsen2023unconstrained, neu2024dealing, domke2023provable}. To manage the unbounded variance permitted by \eqref{eq:bg}, several works have employed Halpern iterations in both convex \cite{neu2024dealing, alacaoglu2025towards} and non-convex settings \cite{fazla2026lower}. Closer to our approach, \cite{domke2023provable} establishes finite-time horizon convergence rates for Projected SGD under the BG condition. Our work broadens this framework by generalizing the analysis to encompass minibatches, preconditioning, and asymptotic regimes. Furthermore, for bounded iterates in the non-convex setting, we recover an analog to the composite stochastic optimization rates established by \cite{ghadimi2016mini} under the stricter bounded variance regime.

\section{Asymptotic analysis for non convex functions and dynamic batching}\label{app:nonconvex}

Handling assumption \eqref{eq:bg} in a non convex setting is a hard problem due to the unbounded nature of the variance. As shown recently in \cite{fazla2026lower}, a way to tackle this issue is to use Halpern iteration \cite{Halpern1967FixedPO}. Such an approach could be the object of future work to treat general VI problems with theoretical guarantees.

As a more restrictive approach, we simplify the problem by supposing that the sequence \(\left(\theta^k\right)_{k\in\mathbb{N}}\) is bounded. Such an assumption ensures that the variance of the gradient estimate along the iterates is ultimately bounded. Indeed, combined to \eqref{eq:bgLambda}, it guarantees that there exists a compact set \(\compact\subset\Theta\) such that for any \(k\in\mathbb{N}\), \(\theta^k\in\compact\) and consequently,
\[\E\left\|\mathbf{g}\left(\theta^k\right)\right\|_\Lambda^2\le a\left\|\theta^k-\theta^*\right\|_{\Lambda^{-1}}^2+b\le a\sup_{\theta\in \compact}\|\theta-\theta^*\|^2_{\Lambda^{-1}}+b:=\bar b.\]
Given this property, it is possible to apply a classical analysis (e.g. \cite{ghadimi2016mini}) to investigate the non convex setting. Before stating the main theorem of this section, we define the gradient mapping \cite{nesterov2013gradient} as
\[G_\gamma:\theta\mapsto\frac{1}{\gamma}\Lambda^{-1}\left(\theta-\mathcal{P}_{\Theta}\left(\theta-\gamma\Lambda\nabla F(\theta)\right)\right).\]
Note that if \(\theta-\gamma\Lambda\nabla F(\theta)\in\Theta\), then the gradient mapping is exactly equal to \(\nabla F(\theta)\).

\begin{theorem}\label{thm:non_convex1}
    Let $F:\R^D\to\R$ be a differentiable and \(L\)-smooth function on a nonempty closed convex set $\Theta$. Suppose that \(F\) has a global minimizer on \(\Theta\) and denote it \(\theta^*\). We assume that \(\Theta\) and \(\Lambda\) satisfy Assumption \ref{ass:block_separable_metric} and that the gradient estimate \(\mathbf{g}\) satisfies \eqref{eq:bgLambda} for some \((a,b,\theta^*)\).
    Let the sequence \(\left(\theta^k\right)_{k\in\mathbb{N}}\) be generated by \eqref{eq:MPSGD} with \(\gamma_k\in\left(0,\frac{1}{L}\right)\) and \(N_k\in\mathbb{N}^*\) for any \(k\in\mathbb{N}\).
    
    If the sequence \(\left(\theta^k\right)_{k\in\mathbb{N}}\) is bounded, then there exists \(\bar b>0\) such that for any \(K\ge0\),
    \[\min_{0\le k\le K}\mathbb{E}\left[\left\|G_{\gamma_k}\left(\theta^k\right)\right\|^2_{\Lambda}\right]\le \frac{4\left(F\left(\theta^0\right)-F\left(\theta^*\right)\right)}{\sum_{k=0}^K\gamma_k(1-\gamma_kL)}+\frac{4\sum_{k=0}^K\frac{\gamma_k}{N_k}}{\sum_{k=0}^K\gamma_k(1-\gamma_kL)}\bar b.\]
    In addition, if \(\sum_{k\in\mathbb{N}}\frac{\gamma_k}{N_k}<+\infty\), then
    \[\sum_{k=0}^K\gamma_k\left(1-\gamma_k L\right)\mathbb{E}\left[\left\|G_{\gamma_k}\left(\theta^k\right)\right\|^2_{\Lambda}\right]<+\infty.\]
\end{theorem}

The proof of Theorem \ref{thm:non_convex1} can be found in Appendix \ref{proof:non_convex}.

\begin{remark}
    An important consequence of Theorem \ref{thm:non_convex1} is that, even when assuming that the sequence \(\left(\theta^k\right)_{k\in\mathbb{N}}\) is bounded, choosing a decreasing step size policy is not enough to ensure that \(\min_{0\le k\le K}\mathbb{E}\left\|G_{\gamma_k}\left(\theta^k\right)\right\|^2\) goes to \(0\). Indeed, the given bound guarantees the minimal composite gradient to vanish only if the batch size \(N_k\) increases asymptotically.
\end{remark}

\section{Experimental Setup for Black-Box Variational Inference}
\label{app:experiment_details}

In this section, we give the setting of the experiment presented in Figure \ref{fig:comp1} and additional graphs.

We consider a Bayesian Logistic Regression model with a Gaussian prior
The target function $f:\mathbb{R}^d \to \mathbb{R}$ is defined as the negative unnormalized log-posterior of a Bayesian Logistic Regression model with a Gaussian prior ($\mathcal{N}(0, \alpha^{-1}I_d)$, for \(\alpha>0\)). Given a design matrix $U \in \mathbb{R}^{M \times d}$ (with rows $u_i^T$) and binary labels $y \in \{0, 1\}^M$, the objective is explicitly formulated as:
\begin{equation}
    f(x) = \frac{\alpha}{2} \|x\|^2 + \sum_{i=1}^M \left[ \log(1 + \exp(u_i^T x)) - y_i (u_i^T x) \right].
\end{equation}
We generate \(U\) with a uniform random noise in \([0,1]\) for each component, and \(y\) with a Bernoulli law of probability \(\frac{1}{2}\). 

The resulting \(f\) is convex and \(L\)-smooth with \(L = \alpha + \frac{1}{4}\sigma_{max}(U)\).

Figure \ref{fig:comp1} is generated for the above problem with \(d=200\), \(M=300\), \(\alpha=0.1\) and the search distribution \(\search\) defined from a Laplace distribution
\[\varphi:r\mapsto\frac{\sqrt{d+1}}{2}\exp(-\sqrt{d+1}|r|).\]
The distribution \(\search\) is standardized and its kurtosis is \(3\frac{d+3}{d+1}\).

For a budget of \(E=5000\) iterations, we run the following methods:
\begin{itemize}
    \item vanilla Projected Stochastic Gradient Descent (PSGD): the size of the batches is constant equal to \(1\) and the stepsize is chosen by optimizing the bound of Theorem \ref{thm:MPSGD}, i.e. \(\gamma = \frac{1}{2L\sqrt{d-1+\kappa(\search)}\sqrt{E}}\). Note that resulting step size is slightly smaller than that given in \cite{domke2023provable}, differing by a factor \(\sqrt{2}\).
    \item scaled PSGD: we incorporate a scaling matrix \(\Lambda\) defined as in the second point of Theorem \ref{cor:BBVI_I}. This allows to choose the theoretically optimal step size \(\gamma = \frac{1}{2L\sqrt{E}}\) which is significantly larger than that in vanilla PSGD.
    \item MPSGD without scaling: we apply exactly the method described in the first point of Theorem \ref{cor:BBVI_I}, choosing \(\Lambda=I_{d+d^2}\), \(\gamma=\frac{1}{2L}\), \(N=\sqrt{d+\kappa(\search)}\sqrt{E}\) and \(K=\sqrt{E}/\sqrt{d+\kappa(\search)}\).
    \item MPSGD with scaling: the second method described in Theorem \ref{cor:BBVI_I}, choosing \(\Lambda\) to balance the variance in \(\mu\) and \(\Sigma\), \(\gamma=\frac{1}{2L}\), \(N=\sqrt{E}\) and \(K=\sqrt{E}\).
\end{itemize}
Each method starts from \(\left(\mu^0,\Sigma^0\right)=\left(0_d,\frac{1}{\sqrt{L}}I_d\right)\).

We provide in Figure \ref{fig:comp2} and Figure \ref{fig:comp3} two other instances of this problem, respectively considering a gaussian base distribution (\(\varphi:r\mapsto\frac{1}{\sqrt{2\pi}}\exp(-r^2/2)\) and \(\kappa(\search)=3\)) with \(M=1000\) and \(\alpha=0.1\), and a uniform base distribution (\(\varphi:r\mapsto\frac{1}{2\sqrt{d+2}}\mathbf{1}_{[-\sqrt{d+2},\sqrt{d+2}]}(r)\) and \(\kappa(\search)=3\frac{d+2}{d+4}\)) with \(M=2000\) and \(\alpha=0.5\). 

\begin{figure}[h]
    \centering
    \includegraphics[width=0.5\linewidth]{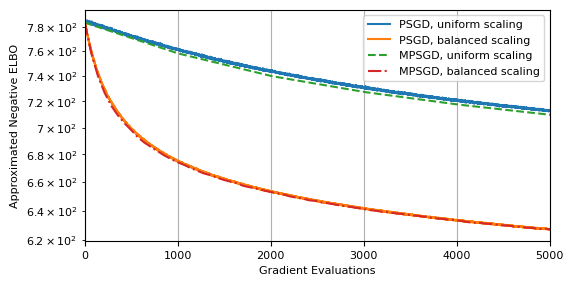}\hfill
    \includegraphics[width=0.5\linewidth]{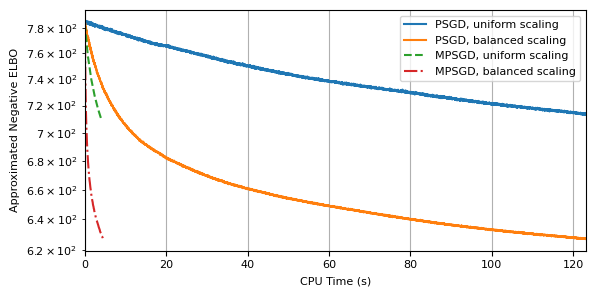}
    \caption{Evolution of the Negative ELBO (approximated with \(100\) samples) along the iterates of PSGD \cite{domke2023provable} (in blue), PSGD with adapted scaling (in orange), MPSGD with \(\Lambda=I_{d+d^2}\) (in green, dashed line) and MPSGD with adapted scaling (see Theorem \ref{cor:BBVI_I}, in red), w.r.t. the number of gradient evaluations (left) and the CPU time (right). We consider a Gaussian search distribution, for \(d=200\).}
    \label{fig:comp2}
\end{figure}

\begin{figure}[h]
    \centering
    \includegraphics[width=0.5\linewidth]{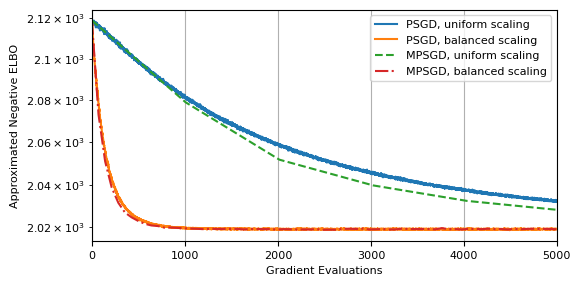}\hfill
    \includegraphics[width=0.5\linewidth]{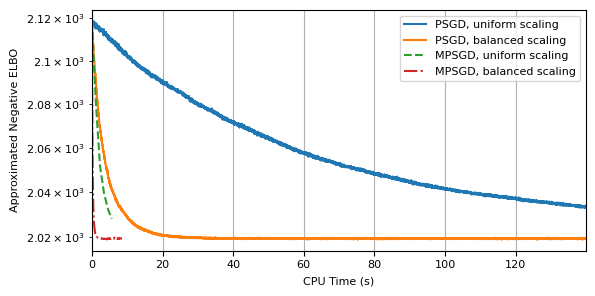}
    \caption{Evolution of the Negative ELBO (approximated with \(100\) samples) along the iterates of PSGD \cite{domke2023provable} (in blue), PSGD with adapted scaling (in orange), MPSGD with \(\Lambda=I_{d+d^2}\) (in green, dashed line) and MPSGD with adapted scaling (see Theorem \ref{cor:BBVI_I}, in red), w.r.t. the number of gradient evaluations (left) and the CPU time (right). We consider a Uniform search distribution, for \(d=200\).}
    \label{fig:comp3}
\end{figure}

The experimemts were performed in Python on a 2,4 GHz Intel Core i5 quad-core laptop with 8 Gb of RAM.

\section{Technical proofs}

\subsection{Proof of Theorem \ref{thm:existence}}\label{proof:existence}

The proof of Theorem \ref{thm:existence} consists in showing that any sublevel set of \(\mathcal{L}\) is compact before applying Weierstrass extreme value theorem.

Let \(\left(\bar\mu,\bar\Sigma\right)\in\R^d\times\mathcal{S}^d_{++}\) and \( V:=\mathcal{L}\left(\bar\mu,\bar\Sigma\right)\). We define the \(V\)-sublevel set as
\[\Omega_V:=\left\{\left(\mu,\Sigma\right)\in\R^d\times\mathcal{S}^d_{++},~\mathcal{L}\left(\mu,\Sigma\right)\le V \right\}.\]
By assumption, we know that \(f\) has a \(\growth\)-growth for some \(\growth>0\), i.e. that there exist \(c_1>0\) and \(c_2\in\R\) such that
\[\forall x\in\R^d,\quad f(x)\ge c_1\|x\|^\growth-c_2.\]
As a consequence and due to the definition of elliptic location scale families, it holds that for any \(\mu\) and \(\Sigma\),
\[\E_{x\sim\search_{\mu,\Sigma}}\left[f(x)\right]=\int_{\R^d}f\left(\mu+\Sigma z\right)\search(z)dz\ge c_1\int_{\R^d}\left\|\mu+\Sigma z\right\|^\growth\search(z)dz-c_2.\]
We distinguish the two following cases:
\begin{itemize}
    \item \textbf{Quadratic growth (\(\growth=2\)).} This case is the most straightforward as by symmetry of \(\search\), we have that \[\int_{\R^d}\langle\mu,\Sigma z\rangle\search(z)dz=0,\]
    and consequently
    \[\begin{aligned}\E_{x\sim\search_{\mu,\Sigma}}\left[f(x)\right]&\ge c_1\|\mu\|^2+c_1\int_{\R^d}\left\|\Sigma z\right\|^2\search(z)dz-c_2\\&=c_1\|\mu\|^2+c_1\left\|\Sigma\right\|^2_F\E_{z\sim\search}\left[\|z\|^2\right]-c_2\\
    &=c_1\|\mu\|^2+c_1d\left\|\Sigma\right\|^2_F-c_2,\end{aligned}\]
    where the last equality comes from the assumption that \(\search\) is standardized.
    \item \textbf{Weak growth (\(\growth>0\)) and \(\search\) log-concave.} By assuming the log-concavity of the search distribution \(\search\) (which implies that of \(\search_{\mu,\Sigma}\)), we can apply the main result from \cite{Latala} (also stated in \cite{milman1986asymptotic,brazitikos2014geometry} for \(\growth\ge1\)) to write that there exists \(C_\growth>0\) (independent from \(\mu\) and \(\Sigma\)) such that
    \[\E_{x\sim\search_{\mu,\Sigma}}\|x\|^\growth\ge C_\growth\left(\E_{x\sim\search_{\mu,\Sigma}}\|x\|^2\right)^{\frac{\growth}{2}}.\]
    Following the computations from the previous case, we know that
    \[\E_{x\sim\search_{\mu,\Sigma}}\|x\|^2\ge\|\mu\|^2+d\|\Sigma\|^2_F,\]
    which implies that
    \[\E_{x\sim\search_{\mu,\Sigma}}\left[f(x)\right]\ge c_1C_\growth \left(\|\mu\|^2+d\|\Sigma\|^2_F\right)^{\frac{\growth}{2}}-c_2.
    \]
    By applying Jensen's inequality, for any \(a>0\), \(b>0\) and \(q\in(0,1)\), \((a+b)^q\ge2^{q-1}(a^q+b^q)\) which leads to
    \[\left(\|\mu\|^2+d\|\Sigma\|^2_F\right)^{\frac{\growth}{2}}\ge 2^{\frac{\growth}{2}-1}\left(\|\mu\|^\growth+d^{\frac{\growth}{2}}\|\Sigma\|^\growth_F\right).\]
    We finally get the following lower bound:
    \[\E_{x\sim\search_{\mu,\Sigma}}\left[f(x)\right]\ge c_1C_\growth2^{\frac{\growth}{2}-1}\left(\|\mu\|^\growth+d^{\frac{\growth}{2}}\|\Sigma\|^\growth_F\right)-c_2. \]
    Note then that by imposing that \(\growth\in(0,2)\) and applying H\"older's inequality, we get that
    \[\left\|\Sigma\right\|^\growth_F=\left(\sum_{i=1}^d\sigma_i^2\right)^\frac{\growth}{2}\ge d^{\frac{\growth}{2}-1}\sum_{i=1}^d\sigma_i^\growth,\]
    where \(\left(\sigma_i\right)_{i=1}^d\) are the singular values of \(\Sigma\).
\end{itemize}
From the above computations, we get that there exist \(\tilde c_1>0\), \(\tilde c_2>0\) and \(\tilde c_3\in\R\) such that
\[\E_{x\sim\search_{\mu,\Sigma}}\left[f(x)\right]\ge \tilde c_1\|\mu\|^\growth+\tilde c_2\sum_{i=1}^d\sigma_i^\growth-\tilde c_3,\]
where \(\left(\sigma_i\right)_{i=1}^d\) are the singular values of \(\Sigma\). As a consequence,
\begin{equation}\label{eq:existence_ineq}\mathcal{L}(\mu,\Sigma)\ge \tilde c_1\|\mu\|^\growth+\tilde c_2\sum_{i=1}^d\sigma_i^\growth-\tilde c_3-\log|\Sigma|= \tilde c_1\|\mu\|^\growth+\sum_{i=1}^d\left(\tilde c_2\sigma_i^\growth-\log \sigma_i\right)-\tilde c_3.\end{equation}
Then, since for any \(\growth>0\), the function \(h:x\mapsto \tilde c_2 x^\growth-\log x\) satisfies \(\lim_{x\to 0^+}h(x)=+\infty\) and \(\lim_{x\to\infty}h(x)=+\infty\), and \(\|\cdot\|^\growth\) is coercive on \(\R^d\), it follows that the following set is bounded
\[\tilde\Omega_V:=\left\{(\mu,\Sigma)\in\R^d\times \mathcal{S}^d_{++},~\tilde c_1\|\mu\|^\growth+\sum_{i=1}^d\left(\tilde c_2\sigma_i^\growth-\log \sigma_i\right)-\tilde c_3\le V\right\}.\]
In addition, the singular values of any \(\Sigma\) in \(\tilde\Omega_V\) are strictly bounded away from \(0\).

Because \(\Omega_V\subseteq\tilde\Omega_V\) (see \eqref{eq:existence_ineq}), \(\Omega_V\) is bounded. Crucially, because the singular values are bounded away from zero, \(\Omega_V\) does not intersect the boundary of the positive definite cone (where \(|\Sigma|=0\)). By continuity of \(\mathcal{L}\) (due to the continuity of \(f\)), \(\Omega_V\) is closed in \(\R^d\times \R^{d\times d}\) and therefore compact. The Weierstrass extreme value theorem then ensures that \[\arg\min_{\left(\mu,\Sigma\right)\in\Omega_V}\mathcal{L}(\mu,\Sigma)\ne\emptyset.\]
We can directly deduce the desired claim.

\subsection{Proof of Theorem \ref{thm:MPSGD}}\label{proof:MBPSGD}

Before proving Theorem \ref{thm:MPSGD}, we give a straightforward lemma for bounding the variance of a minibatch gradient estimate.

\begin{lemma}\label{lem:MC1}
    Let \(F\) be an \(L\)-smooth function on \(\Theta\) with minimizer \(\theta^*\). Let \(\mathbf{g}\) be a gradient estimate of\(\nabla F\) satisfying \eqref{eq:bgLambda} with parameters \((a, b,\theta^*)\). 
    Then, the second moment of the Monte Carlo estimator \(\mathbf{g}^N\) defined in \eqref{eq:MCg} satisfies:
    \[
        \mathbb{E}\left[\left\|\mathbf{g}^N(\theta)\right\|_{\Lambda}^2\right] \le \frac{a}{N}\|\theta-\theta^*\|_{\Lambda^{-1}}^2 + \frac{b}{N} + \frac{N-1}{N}\|\nabla F(\theta)\|_\Lambda^2.
    \]
\end{lemma}
\begin{proof}
    By definition of \(\mathbf{g}^N\), we have that
    \[\E\left\|\mathbf{g}^N(\theta)\right\|_\Lambda^2=\frac{1}{N^2}\E\left\|\sum_{i=1}^N\mathbf{g}(\theta)\right\|_\Lambda^2=\frac{1}{N^2}\sum_{i=1}^N\E\left\|\mathbf{g}(\theta)\right\|_\Lambda^2+\frac{1}{N^2}\sum_{i=1}^N\sum_{j\ne i}\E\langle\mathbf{g}(\theta),\Lambda\mathbf{g}(\xi_j,\theta)\rangle.\]
    Since the samples \(\xi_1, \dots, \xi_N\) are i.i.d. and \(\mathbf{g}(\theta)\) is unbiased, it follows that
    \[\E\left\|\mathbf{g}^N(\theta)\right\|_\Lambda^2=\frac{1}{N}\E\left\|\mathbf{g}(\theta)\right\|_\Lambda^2+\frac{N-1}{N}\left\|\E\left[\mathbf{g}(\theta)\right]\right\|_\Lambda^2=\frac{1}{N}\E\left\|\mathbf{g}(\theta)\right\|_\Lambda^2+\frac{N-1}{N}\left\|\nabla F(\theta)\right\|_\Lambda^2.\]
    By definition of quadratically bounded estimates, we get the desired claim.\qed
\end{proof}

We can now give a proof for Theorem \ref{thm:MPSGD}. We follow the same steps as in the proof of \cite[Theorem~11]{domke2023provable} and exploit Lemma \ref{lem:MC1}. 

By the definition of MPSGD, we have:
\begin{align*}
\left\| \theta^{k+1}-\theta^{*} \right\|_{\Lambda^{-1}}^{2} & = \left\| \mathcal{P}_{\Theta}(\theta^{k}-\gamma_k \mathbf{g}^{N_k}\left(\theta^k\right))-\mathcal{P}_{\Theta}(\theta^{*}) \right\|_{\Lambda^{-1}}^{2} \\
 & \leq \left\| \theta^{k}-\theta^{*}-\gamma_k \mathbf{g}^{N_k}\left(\theta^k\right) \right\|_{\Lambda^{-1}}^{2},
\end{align*}
where we used Assumption \ref{ass:block_separable_metric} and the non-expansivity property of projections onto convex sets. Choosing a fixed step size \(\gamma\in\left(0,\frac{1}{L}\right]\), taking the expectation conditioned on $\theta^{k}$ and using the convexity and smoothness of \(F\),
\begin{align*}
\mathbb{E}\left[ \left\| \theta^{k+1}-\theta^{*} \right\|_{\Lambda^{-1}}^{2} \mid \theta^{k} \right] & \leq \left\| \theta^{k}-\theta^{*} \right\|_{\Lambda^{-1}}^{2} - 2\gamma \left\langle \nabla F\left(\theta^k\right), \theta^{k}-\theta^{*} \right\rangle + \gamma^{2} \mathbb{E}\left[ \left\| \Lambda \mathbf{g}^{N_k}\left(\theta^k\right) \right\|_{\Lambda^{-1}}^{2} \mid \theta^{k} \right] \\
&\le\left\| \theta^{k}-\theta^{*} \right\|_{\Lambda^{-1}}^{2} - 2\gamma(F\left(\theta^k\right)-F^*)-\frac{\gamma}{L}\left\|\nabla F\left(\theta^k\right)\right\|^2 + \gamma^{2} \mathbb{E}\left[ \left\| \mathbf{g}^{N_k}\left(\theta^k\right) \right\|_{\Lambda}^{2} \mid \theta^{k} \right]
\end{align*}
Then, by applying Lemma \ref{lem:MC1} and taking a fixed minbatch size along iterations \(N_k=N\),
\[\begin{aligned}\mathbb{E}\left[ \left\| \theta^{k+1}-\theta^{*} \right\|_{\Lambda^{-1}}^{2} \mid \theta^{k} \right]
  &\leq \left\| \theta^{k}-\theta^{*} \right\|_{\Lambda^{-1}}^{2} - 2\gamma (F\left(\theta^k\right)-F(\theta^{*})) + \frac{a}{N} \gamma^{2} \left\| \theta^{k}-\theta^{*} \right\|_{\Lambda^{-1}}^{2} + \frac{b}{N} \gamma^{2}\\&\qquad+~\gamma\left(\gamma\frac{N-1}{N}\left\|\nabla F\left(\theta^k\right)\right\|_\Lambda^2-\frac{1}{L}\left\|\nabla F\left(\theta^k\right)\right\|^2\right).\end{aligned}\]
From the assumptions on \(\Lambda\), we have that \(\left\|\nabla F\left(\theta^k\right)\right\|_\Lambda^2\le\left\|\nabla F\left(\theta^k\right)\right\|^2\). As a consequence, for \(\gamma<\frac{N}{(N-1)L}\) the last term is negative. Taking the expectation along iterates gives
\begin{align*}
2\gamma \mathbb{E}\left[ F\left(\theta^k\right) - F(\theta^{*}) \right] \leq \left( 1 + \frac{a}{N} \gamma^{2} \right) \mathbb{E}\left[ \left\| \theta^{k}-\theta^{*} \right\|_{\Lambda^{-1}}^{2} \right] - \mathbb{E}\left[ \left\| \theta^{k+1}-\theta^{*} \right\|_{\Lambda^{-1}}^{2} \right] + \frac{b}{N} \gamma^{2}.
\end{align*}
Let $C = 1 + \frac{a}{N} \gamma^{2}$. We define weights $\alpha_{k} = C^{-k} = \tau^k$ to create a telescopic sum. Multiplying by $\alpha_{k+1}$ and summing from $k=0$ to $K-1$:
\begin{align*}
2\gamma \sum_{k=0}^{K-1} \alpha_{k+1} \mathbb{E}\left[ F\left(\theta^k\right) - F(\theta^{*}) \right] & \leq \sum_{k=0}^{K-1} \left( \alpha_{k} \mathbb{E}\left[ \left\| \theta^{k}-\theta^{*} \right\|_{\Lambda^{-1}}^{2} \right] - \alpha_{k+1} \mathbb{E}\left[ \left\| \theta^{k+1}-\theta^{*} \right\|_{\Lambda^{-1}}^{2} \right] \right) + \frac{b\gamma^2}{N} \sum_{k=0}^{K-1} \alpha_{k+1} \\
 & = \alpha_{0} \left\| \theta^{0}-\theta^{*} \right\|_{\Lambda^{-1}}^{2} - \alpha_{K} \mathbb{E}\left[ \left\| \theta^{K}-\theta^{*} \right\|_{\Lambda^{-1}}^{2} \right] + \frac{b\gamma^2}{N} \sum_{k=0}^{K-1} \alpha_{k+1}.
\end{align*}
Using $\alpha_0 = 1$ and Jensen's inequality for the weighted average $\bar{\theta}^{K}$:
\begin{align*}
\mathbb{E}\left[ F(\bar{\theta}^{K}) - F(\theta^{*}) \right] \leq \frac{\left\| \theta^{0}-\theta^{*} \right\|_{\Lambda^{-1}}^{2}}{2\gamma \sum_{k=0}^{K-1} \alpha_{k+1}} + \frac{b\gamma}{2N}.
\end{align*}
Using the geometric series sum for $\alpha_{k+1} = \tau^{k+1}$:
\[
\sum_{k=0}^{K-1} \tau^{k+1} = \frac{1-\tau^K}{C-1} = \frac{1-\tau^K}{aN^{-1}\gamma^2}.
\]
Substituting this back into the inequality:
\begin{align*}
\mathbb{E}\left[ F(\bar{\theta}^{K}) - F(\theta^{*}) \right] & \leq \frac{aN^{-1}\gamma^2}{1-\tau^K} \frac{\left\| \theta^{0}-\theta^{*} \right\|_{\Lambda^{-1}}^{2}}{2\gamma} + \frac{b\gamma}{2N} \\
& = \frac{\gamma}{2N} \left( \frac{a \left\| \theta^{0}-\theta^{*} \right\|_{\Lambda^{-1}}^{2}}{1-\tau^K} + b \right).
\end{align*}
Then, as for any \(K\in\mathbb{N}^*\) and \(x>0\),
\[\frac{1}{1-(1+x)^{-K}}\le1+\frac{1}{Kx},\]
we can deduce that
\begin{align*}
\mathbb{E}\left[ F(\bar{\theta}^{K}) - F(\theta^{*}) \right] & \leq  \frac{\gamma}{2N} \left( a \left\| \theta^{0}-\theta^{*} \right\|_{\Lambda^{-1}}^{2} + b \right)+\frac{1}{2\gamma K}\left\| \theta^{0}-\theta^{*} \right\|_{\Lambda^{-1}}^{2} .
\end{align*}

\subsection{Proof of Theorem \ref{thm:asymptotic}}\label{app:asymptotic}

The proof of Theorem \ref{thm:asymptotic} can be directly obtained by applying the classical Robbins-Siegmund \cite{robbins1971convergence} theorem stated below.

\begin{lemma}[Robbins-Siegmund]\label{thm:robbins_siegmund}
Let \((\mathcal{F}_k)_{k \in \mathbb{N}}\) be a filtration defined on a probability space \((\Omega, \mathcal{F}, \mathbb{P})\). Let \((V_k)_{k \in \mathbb{N}}\), \((a_k)_{k \in \mathbb{N}}\), \((b_k)_{k \in \mathbb{N}}\), and \((c_k)_{k \in \mathbb{N}}\) be sequences of non-negative random variables adapted to \(\mathcal{F}_k\). Assume that for all \(k \in \mathbb{N}\), the following inequality holds almost surely:
\[
    \mathbb{E}[V_{k+1} \mid \mathcal{F}_k] \leq (1 + a_k)V_k + b_k - c_k.
\]
If \(\sum_{k=0}^\infty a_k < \infty\) and \(\sum_{k=0}^\infty b_k < \infty\) almost surely, then:
\begin{enumerate}
    \item \(\lim_{k \to \infty} V_k = V_\infty\) exists and is finite almost surely.
    \item \(\sum_{k=0}^\infty c_k < \infty\) almost surely.
\end{enumerate}
    
\end{lemma}

The first steps of the proof are identical to that of Theorem \ref{thm:MPSGD}, without taking \(\left(\gamma_k\right)_{k\in\mathbb{N}}\) and \(\left(N_k\right)_{k\in\mathbb{N}}\) as constants. Applying the same steps, it follows that
\[\mathbb{E}\left[ \left\| \theta^{k+1}-\theta^{*} \right\|_{\Lambda^{-1}}^{2} \mid \theta^k \right]\le \left( 1 + \frac{a}{N_k} \gamma_k^{2} \right) \left\| \theta^{k}-\theta^{*} \right\|_{\Lambda^{-1}}^{2} + \frac{b}{N_k} \gamma_k^{2} -2\gamma_k \left( F\left(\theta^k\right) - F(\theta^{*}) \right).\]
Let $\tau_k = \left(1 + \frac{a}{N_k} \gamma_k^{2}\right)^{-1}$.  We define weights $\alpha_0=1$ and $\alpha_{k} = \prod_{i=0}^{k-1} \tau_i$ for $k\geq 1$,  to create a telescopic sum. Multiplying by $\alpha_{k+1}$ and summing from $k=0$ to $K-1$:
\begin{align*}
2 \sum_{k=0}^{K-1} \gamma_k\alpha_{k+1} \mathbb{E}\left[ F\left(\theta^k\right) - F(\theta^{*}) \right] & \leq \sum_{k=0}^{K-1} \left( \alpha_{k} \mathbb{E}\left[ \left\| \theta^{k}-\theta^{*} \right\|_{\Lambda^{-1}}^{2} \right] - \alpha_{k+1} \mathbb{E}\left[ \left\| \theta^{k+1}-\theta^{*} \right\|_{\Lambda^{-1}}^{2} \right] \right) +  \sum_{k=0}^{K-1} \frac{b\gamma_k^2}{N_k}\alpha_{k+1} \\
 & = \alpha_{0} \left\| \theta^{0}-\theta^{*} \right\|_{\Lambda^{-1}}^{2} - \alpha_{K} \mathbb{E}\left[ \left\| \theta^{K}-\theta^{*} \right\|_{\Lambda^{-1}}^{2} \right] + \sum_{k=0}^{K-1} \frac{b\gamma_k^2}{N_k} \alpha_{k+1}.
\end{align*}
Using $\alpha_0 = 1$ and Jensen's inequality for the weighted average $\bar{\theta}^{K}$:
\begin{align*}
\mathbb{E}\left[ F(\bar{\theta}^{K}) - F(\theta^{*}) \right] \leq \frac{\left\| \theta^{0}-\theta^{*} \right\|_{\Lambda^{-1}}^{2}}{2 \sum_{k=0}^{K-1} \gamma_k\alpha_{k+1}} +  \frac{\sum_{k=0}^{K-1} \frac{b\gamma_k^2}{N_k} \alpha_{k+1}}{2 \sum_{k=0}^{K-1} \gamma_k \alpha_{k+1}}.
\end{align*}
Let $S=\sum_{k=0}^{+\infty}\gamma_k^2/N_k$. Thanks to the inequality $1+t\leq e^{t}$ (and thus $(1+t)^{-1}\geq e^{-t}$) which is valid for any $t\in\mathbb{R}$, we have that for every $k\in\mathbb{N}$
\[
\prod_{i=0}^{k} \left(1+\frac{a\gamma_i^2}{N_i}\right)^{-1}\geq \prod_{i=0}^{k} \exp\left(-\frac{a\gamma_i^2}{N_i}\right)= \exp\left(-\sum_{i=0}^{k}\frac{a\gamma_i^2}{N_i}\right).  
 \]
 Hence,
\[
\sum_{k=0}^{K-1} \gamma_k  \alpha_{k+1} \geq\sum_{k=0}^{K-1} \gamma_k\exp\left(-\sum_{i=0}^{k}\frac{a\gamma_i^2}{N_i}\right) 
\]
In addition note that 
\[\sum_{k=0}^{K-1} \frac{a\gamma_k^2}{N_k} \alpha_{k+1} =\sum_{k=0}^{K-1} \left(1+\frac{a\gamma_k^2}{N_k} -1\right){\alpha_{k+1}} = \sum_{k=0}^{K-1} ({\alpha_k}-\alpha_{k+1})\leq {\alpha_0}=1.  \]
Substituting this back into the inequality:
\begin{align*}
 \mathbb{E}\left[ F(\bar{\theta}^{K}) - F(\theta^{*}) \right] & \leq \frac{1}{2 \sum_{k=0}^{K-1} \gamma_k\left(\exp\left(-\sum_{i=0}^{k}\frac{a\gamma_i^2}{N_i}\right)\right)} \left( \left\| \theta^{0}-\theta^{*} \right\|_{\Lambda^{-1}}^{2} + \frac{b}{a}\right)  
 \end{align*}
Finally, to prove almost sure convergence, applying Robbins-Siegmund theorem (Lemma \ref{thm:robbins_siegmund}), it immediately holds that if \(\sum_{k\in\mathbb{N}}\frac{\gamma_k^2}{N_k}<\infty\), then 
\[\sum_{k\in\mathbb{N}}\gamma_k \left(F\left(\theta^k\right) - F(\theta^{*}) \right)<\infty,\quad \text{a.s.}\]

\subsection{Proof of Theorem \ref{thm:non_convex1}}\label{proof:non_convex}

We start by leveraging the smoothness of \(F\) to write
\begin{equation}\label{eq:smoothness_nc}F\left(\theta^{k+1}\right)\le F\left(\theta^k\right)+\langle \theta^{k+1}-\theta^k,\nabla F\left(\theta^k\right)\rangle+\frac{L}{2}\left\|\theta^{k+1}-\theta^k\right\|^2.\end{equation}
Since Assumption \ref{ass:block_separable_metric} holds, we have that the projection onto \(\Theta\) w.r.t. the metric \(\Lambda^{-1}\) is equal to that w.r.t. the euclidean metric. As a consequence, the property of the projection onto a convex set gives us that
\[\left\langle\theta^k-\gamma_k\Lambda \mathbf{g}^{N_k}(\theta^k)-\theta^{k+1},\theta^k-\theta^{k+1}\right\rangle_{\Lambda^{-1}}\le 0,\]
and consequently,
\[\left\|\theta^{k+1}-\theta^k\right\|^2_{\Lambda^{-1}}\le\gamma_k\left\langle\mathbf{g}^{N_k}(\theta^k),\theta^k-\theta^{k+1}\right\rangle.\]
It follows that
\[\begin{aligned}\langle \theta^{k+1}-\theta^k,\nabla F\left(\theta^k\right)\rangle&=\langle \theta^{k+1}-\theta^k,\nabla F\left(\theta^k\right)-\mathbf{g}^{N_k}(\theta^k)\rangle+\langle \theta^{k+1}-\theta^k,\mathbf{g}^{N_k}(\theta^k)\rangle\\
&\le \frac{1}{2\gamma_k}\left\|\theta^{k+1}-\theta^k\right\|^2_{\Lambda^{-1}}+\frac{\gamma_k}{2}\left\|\nabla F\left(\theta^k\right)-\mathbf{g}^{N_k}(\theta^k)\right\|^2_\Lambda+\langle \theta^{k+1}-\theta^k,\mathbf{g}^{N_k}(\theta^k)\rangle\\
&\le-\frac{1}{2\gamma_k}\left\|\theta^{k+1}-\theta^k\right\|^2_{\Lambda^{-1}}+\frac{\gamma_k}{2}\left\|\nabla F\left(\theta^k\right)-\mathbf{g}^{N_k}(\theta^k)\right\|^2_\Lambda.\end{aligned}\]
By rewriting \eqref{eq:smoothness_nc} with the above inequality, we have that
\[F\left(\theta^{k+1}\right)\le F\left(\theta^k\right)-\frac{1}{2\gamma_k}\left(1-\gamma_k L\right)\left\|\theta^{k+1}-\theta^k\right\|^2_{\Lambda^{-1}}+\frac{\gamma_k}{2}\left\|\nabla F\left(\theta^k\right)-\mathbf{g}^{N_k}(\theta^k)\right\|^2_\Lambda,\]
where we use the inequality \(\left\|\theta^{k+1}-\theta^k\right\|^2\le\left\|\theta^{k+1}-\theta^k\right\|^2_{\Lambda^{-1}}\).
The next step is to link \(\theta^{k+1}-\theta^k\) to the gradient mapping \(G_{\gamma^k}\left(\theta^k\right)\).
It is straightforward to show that
\[\left\|G_{\gamma_k}\left(\theta^k\right)\right\|^2_{\Lambda}\le\frac{2}{\gamma_k^2}\left\|\theta^{k+1}-\theta^k\right\|^2_{\Lambda^{-1}}+\frac{2}{\gamma_k^2}\left\|\theta^{k+1}-\mathcal{P}_{\Theta}\left(\theta^k-\gamma_k\Lambda\nabla F\left(\theta^k\right)\right)\right\|^2_{\Lambda^{-1}},\]
and by non expansivity of the projection,
\[\left\|G_{\gamma_k}\left(\theta^k\right)\right\|^2_{\Lambda}\le\frac{2}{\gamma_k^2}\left\|\theta^{k+1}-\theta^k\right\|^2_{\Lambda^{-1}}+2\left\|\mathbf{g}^{N_k}(\theta^k)-\nabla F\left(\theta^k\right)\right\|^2_{\Lambda}.\]
Since by assumption \(\gamma_k<\frac{1}{L}\) for any \(k\in\mathbb{N}\), we can write that
\begin{equation}\label{eq:checkpoint1}F\left(\theta^{k+1}\right)\le F\left(\theta^k\right)-\frac{\gamma_k}{4}\left(1-\gamma_k L\right)\left\|G_{\gamma_k}\left(\theta^k\right)\right\|^2_{\Lambda^{-1}}+\frac{\gamma_k}{2}(2-\gamma_k L)\left\|\nabla F\left(\theta^k\right)-\mathbf{g}^{N_k}(\theta^k)\right\|^2_\Lambda.\end{equation}
The boundedness assumption on the sequence \(\left(\theta^k\right)_{k\in\mathbb{N}}\) implies that there exists a compact set \(\compact\subset \Theta\) such that \(\theta^k\in \compact\) for any \(k\in\mathbb{N}\). In addition, since \(\mathbf{g}\) satisfies \eqref{eq:bgLambda}, there exist \(\bar b>0\) depending on \(\compact\) such that for any \(k\in\mathbb{N}\)
\[\mathbb{E}\left\|\mathbf{g}\left(\theta^k\right)\right\|_\Lambda^2\le\bar b,\]
and consequently
\[\mathbb{E}\left\|\mathbf{g}^{N_k}\left(\theta^k\right)\right\|_\Lambda^2\le\frac{\bar b}{N_k}+\frac{N_k-1}{N_k}\left\|\nabla F\left(\theta^k\right)\right\|^2_\Lambda.\]
Hence,
\[\mathbb{E}\left\|\mathbf{g}^{N_k}\left(\theta^k\right)-\nabla F\left(\theta^k\right)\right\|_\Lambda^2\le\frac{\bar b}{N_k}-\frac{1}{N_k}\left\|\nabla F\left(\theta^k\right)\right\|^2_\Lambda.\]
By taking the expectation w.r.t. \(\theta^k\) of \eqref{eq:checkpoint1}, it follows that
\[\mathbb{E}\left[F\left(\theta^{k+1}\right)|\theta^k\right]\le F\left(\theta^k\right)-\frac{\gamma_k}{4}\left(1-\gamma_k L\right)\left\|G_{\gamma_k}\left(\theta^k\right)\right\|^2_{\Lambda^{-1}}+\frac{\gamma_k}{2N_k}(2-\gamma_k L)\bar b.\]
By then summing on the iterates \(k\) from \(0\) to \(K\) and "taking the expectation over the all sequence",
\begin{equation}\label{eq:telescopic_nc}\mathbb{E}\left[\sum_{k=0}^K\frac{\gamma_k}{4}\left(1-\gamma_k L\right)\left\|G_{\gamma_k}\left(\theta^k\right)\right\|^2_{\Lambda}\right]\le F\left(\theta^0\right)-\mathbb{E}\left[F\left(\theta^{K+1}\right)\right] +\sum_{k=0}^K\frac{\gamma_k}{N_k}\bar b.\end{equation}
We can therefore conclude that
\[\min_{0\le k\le K}\mathbb{E}\left[\left\|G_{\gamma_k}\left(\theta^k\right)\right\|^2_{\Lambda}\right]\le \frac{4\left(F\left(\theta^0\right)-\inf_{\theta\in\Theta}F(\theta)\right)}{\sum_{k=0}^K\gamma_k(1-\gamma_kL)}+\frac{4\sum_{k=0}^K\frac{\gamma_k}{N_k}}{\sum_{k=0}^K\gamma_k(1-\gamma_kL)}\bar b.\]
From \eqref{eq:telescopic_nc}, we can also deduce that as long as \(\sum_{k\in\mathbb{N}}\frac{\gamma_k}{N_k}<+\infty\), we have that
\[\sum_{k=0}^K\frac{\gamma_k}{4}\left(1-\gamma_k L\right)\mathbb{E}\left[\left\|G_{\gamma_k}\left(\theta^k\right)\right\|^2_{\Lambda^{-1}}\right]<+\infty.\]

\subsection{Proof of Theorem \ref{cor:BBVI_I}}\label{app:cor_VI}

Given the assumptions satisfied by \(f\) and \(\search\), Corollary \ref{cor:existence} ensures that the loss \(\mathcal{L}\) has a global minimizer \(\left(\mu^*,\Sigma^*\right)\in\Theta_L\) and Proposition \ref{prop:properties} guarantees that \(\mathcal{L}\) is convex and \(2L\)-smooth on \(\Theta_L\). In addition, following \cite[Lemma~19]{domke2023provable}, \(\Theta_L\) is a closed convex set. It is therefore sufficient to show that assumption \eqref{eq:bgLambda} is verified in this setting for each choice of \(\Lambda\) in order to apply Theorem \ref{thm:MPSGD} and conclude.

\paragraph{First case:} Let \(\Lambda=I_{d+d^2}\). According to Theorem \ref{thm:g_variance} and Theorem \ref{thm:g_Sigma_ent}, the variance of the gradient estimate \(\gLmuS=\left(\gLmu,\gLS\right)\) can be upper bounded in the following way
\[
\begin{aligned}
    \E_{z\sim\search}\left[\left\|\gLmuS(z)\right\|^2\right]&=\E_{z\sim\search}\left[\left\|\gLmu(z)\right\|^2\right]+\E_{z\sim\search}\left[\left\|\gLS(z)\right\|_F^2\right]\\&\le \left(1+2d\right)L^2\left\|\mu-\tilde\mu\right\|^2+2(d+\kappa(\search))L^2\left\|\Sigma\right\|^2_F+2dL,
\end{aligned}
\]
where \(\tilde\mu\) is a (global) minimizer of \(f\) (which is guaranteed to exist by \(\growth\)-growth of \(f\) and convexity). Since \(2d+1<2(d+\kappa(\search))\) for any standardized \(\search\) and using Young's inequality, it follows that 
\[
\begin{aligned}
    \E_{z\sim\search}\left[\left\|\gLmuS(z)\right\|^2\right]&\le 2(d+\kappa(\search))L^2\left\|\mu-\tilde\mu\right\|^2+2(d+\kappa(\search))L^2\left\|\Sigma\right\|^2_F+2dL\\
    &\le a\left(\left\|\mu-\tilde\mu\right\|^2+\left\|\Sigma-\Sigma^*\right\|^2_F\right)+b,
\end{aligned}
\]
where \(a=4(d+\kappa(\search))L^2\) and \[b=a\left(\left\|\mu^*-\tilde\mu\right\|^2+\left\|\Sigma^*\right\|^2_F\right)+2dL.\]
We can conclude that this gradient estimate satisfies \eqref{eq:bgLambda} with parameters \(\left(a,b,\left(\mu^*,\Sigma^*\right)\right)\).

\paragraph{Second case:} Let \(\Lambda\) be defined as \eqref{eq:Lambda_VI} with \(\lambda_\mu=1\) and \(\lambda_\Sigma=\frac{1}{2(d-1+\kappa(\search))}\). We follow the same steps as in the previous case. Theorem \ref{thm:g_variance} and Theorem \ref{thm:g_Sigma_ent} guarantee that
\[
\begin{aligned}
    \E_{z\sim\search}\left[\left\|\gLmuS(z)\right\|_\Lambda^2\right]&=\E_{z\sim\search}\left[\left\|\gLmu(z)\right\|^2\right]+\lambda_\Sigma\E_{z\sim\search}\left[\left\|\gLS(z)\right\|_F^2\right]\\&\le 2L^2\left\|\mu-\tilde\mu\right\|^2+2L^2\left\|\Sigma\right\|^2_F+L,
\end{aligned}
\]
where \(\tilde\mu\) is a (global) minimizer of \(f\). By applying Young's inequality, it follows that 
\[
    \E_{z\sim\search}\left[\left\|\gLmuS(z)\right\|_\Lambda^2\right]\le a\left(\left\|\mu-\tilde\mu\right\|^2+\left\|\Sigma-\Sigma^*\right\|^2_F\right)+b,
\]
where \(a=4L^2\) and \[b=a\left(\left\|\mu^*-\tilde\mu\right\|^2+\left\|\Sigma^*\right\|^2_F\right)+L.\]
Note that since \(\lambda_\Sigma<1\), it follows that \[\left\|\mu-\tilde\mu\right\|^2+\left\|\Sigma-\Sigma^*\right\|^2_F\le\left\|\left(\mu,\Sigma\right)-\left(\mu^*,\Sigma^*\right)\right\|^2_{\Lambda^{-1}},\]
and consequently, \(\gLmuS\) satisfies \eqref{eq:bgLambda} with parameters \(\left(a,b,\left(\mu^*,\Sigma^*\right)\right)\).

\paragraph{Choosing \(N\) and \(K\)} The calibration of the batch size and the number of iterations follows the idea discussed in Remark \ref{rmk:optimize_bound}. For clarity, we rather choose to set \(N=\alpha\sqrt{E}\) and \(K=\sqrt{E}/\alpha\) with \(\alpha=\gamma\sqrt{a}\). This means that we only optimize the bound w.r.t. the term in front of \(\left\|\theta^0-\theta^*\right\|^2_{\Lambda^{-1}}\) in Theorem \ref{thm:MPSGD}. Applying the theorem with the corresponding constant values lead to the desired claim.

\section{Auxiliary results}
\label{app:standardized}

\begin{proof}[Proof of Proposition \ref{prop:standardized}]

Since \(\base\) is symmetric, we automatically get that \(\E_{z\sim\search}[z_i]=0\). For the second equality, we can write that
\[\E_{z\sim\search}[z_i^2]=\normC\int_{\R^d}z_i^2\base(\|z\|)dz,\]
and using polar coordinates,
\[\E_{z\sim\search}[z_i^2]=\normC\int_\Sd\int_0^\infty\theta_i^2r^{d+1}\base(r)drd\theta.\]
By injecting the expression of \(\normC\) \eqref{eq:def_C} and since \(\int_\Sd\theta_i^2d\theta=\frac{\VolS}{d}\),
\[\E_{z\sim\search}[z_i^2]=\frac{1}{d}\frac{\int_0^\infty r^{d+1}\base(r)dr}{\int_0^\infty r^{d-1}\base(r)dr}.\]
The assumption of the proposition guarantees that \(\E_{z\sim\search}[z_i^2]=1\).

\end{proof}

\begin{proposition}\label{prop:3}
    Let \(F:\R^D\to\R\) be a differentiable function. Suppose \(\Lambda\) and \(\Theta\) satisfy Assumption \ref{ass:block_separable_metric}, and consider a gradient estimate partitioned into \(m\) blocks \(g = ({g}_{(1)}, \dots, {g}_{(m)})\). If, for every \(i \in [m]\), the block estimate is unbiased, \(\E[g_{(i)}(\theta)] = \nabla F(\theta)_i\), and satisfies
    \[ \E\|g_{(i)}(\theta)\|^2 \le \sum_{j=1}^m a_{i,j} \|\theta_j-\theta_j^*\|^2 + b_i, \]
    then \(g\) satisfies \eqref{eq:bgLambda} with parameters \((a,b,\theta^*)\), where \(a = \max_{j\in[m]} \lambda_j \sum_{i=1}^m \lambda_i a_{i,j}\) and \(b = \sum_{i=1}^m \lambda_i b_i\).\end{proposition}

\section{Deriving almost sure asymptotic rates}\label{app:asymptotic_rates}

In this section, we give further details on the almost sure asymptotic rates obtained in Section \ref{sec:asymptotic}. Recall that according to Theorem \ref{thm:asymptotic}, if \(\sum_{k\in\mathbb{N}}\frac{\gamma_k^2}{N_k}<\infty\), then as \(K\) goes to infinity,
\[\min_{k\in[K]}\left(F\left(\theta^k\right)-F\left(\theta^*\right)\right)=\mathcal{O}\left(\left(\sum_{k=1}^K \gamma_k\right)^{-1}\right)\quad a.s.\]
Note that in this setting, the total number of gradient evaluations after \(K\) iterations is \(E=\sum_{k=1}^K N_k\). We distinguish three different strategies.

\paragraph{Strategy 1:} Set \(\gamma_k\asymp k^{-\frac{1}{2}-\varepsilon}\) and \(N_k\asymp1\) for \(\varepsilon>0\). 

We can directly get that \(\frac{\gamma_k^2}{N_k}\asymp k^{-1-2\varepsilon}\). Since \(\varepsilon>0\), we get that the summability condition is valid:
\[\sum_{k\in\mathbb{N}}\frac{\gamma_k^2}{N_k}<\infty.\]

In addition, we have that
\[\sum_{k=1}^K\gamma_k\asymp K^{\frac{1}{2}-\varepsilon}.\]
Since in this setting, the number of gradient evaluations \(E\) is proportional to the number of iterations \(K\), it follows that 
\[\left(\sum_{k=1}^K \gamma_k\right)^{-1}\asymp E^{-\frac{1}{2}+\varepsilon}.\]

\paragraph{Strategy 2:} Set \(\gamma_k\asymp k^{-\frac{1}{2}}\) and \(N_k\asymp k^\sigma\) for \(\sigma>0\).

We have that \(\frac{\gamma_k^2}{N_k}\asymp k^{-1-\sigma}\) which is summable since \(\sigma>0\).

In addition, it holds that
\[\sum_{k=1}^K\gamma_k\asymp K^{\frac{1}{2}},\]
and
\[E=\sum_{k=1}^KN_k\asymp K^{\sigma+1}.\]
It follows that 
\[\left(\sum_{k=1}^K \gamma_k\right)^{-1}\asymp K^{-\frac{1}{2}}\asymp E^{-\frac{1}{2(\sigma+1)}}.\]

\paragraph{Strategy 3:} Set \(\gamma \asymp 1\) and \(N_k\asymp k^{1+\nu}\) for \(\nu>0\).

It is straightforward that \(\frac{\gamma_k^2}{N_k}\asymp k^{-1-\nu}\) which is summable since \(\nu>0\).

We also have that
\[\sum_{k=1}^K\gamma_k\asymp K,\]
and
\[E=\sum_{k=1}^KN_k\asymp K^{\nu+2}.\]
Consequently,
\[\left(\sum_{k=1}^K \gamma_k\right)^{-1}\asymp K\asymp E^{-\frac{1}{2+\nu}}.\]


\end{document}